\documentclass{article}

\usepackage{microtype}
\usepackage{graphicx}
\usepackage{booktabs} %
\usepackage{dsfont}
\usepackage{caption}
\usepackage{subcaption}
\usepackage{lipsum}
\usepackage{adjustbox}
\usepackage{multirow}
\usepackage{makecell}
\usepackage{tablefootnote}

\usepackage{basic_commands}

\usepackage[accepted]{icml2024}

\definecolor{myred}{HTML}{F54254}
\definecolor{myblue}{HTML}{598BE7}
\definecolor{mydarkblue}{HTML}{385492}

\newcommand{\cutsectionup}{\vspace{-6pt}}
\newcommand{\cutsectiondown}{\vspace{-4pt}}
\newcommand{\cutsubsectionup}{\vspace{-5pt}}
\newcommand{\cutsubsectiondown}{\vspace{-4pt}}

\expandafter\def\expandafter\normalsize\expandafter{%
    \normalsize
    \setlength\abovedisplayskip{5pt}
    \setlength\belowdisplayskip{5pt}
    \setlength\abovedisplayshortskip{0pt}
    \setlength\belowdisplayshortskip{0pt}
}

\usepackage{amsmath}
\usepackage{amssymb}
\usepackage{mathtools}
\usepackage{amsthm}

\usepackage[capitalize,noabbrev]{cleveref}

\hypersetup{urlcolor=myred,citecolor=myblue,linkcolor=myblue}
\newcommand{\hilpwebsite}{\url{https://seohong.me/projects/hilp/}}
\newcommand{\hilpcode}{\url{https://github.com/seohongpark/HILP}}
\newcommand{\hilpanonrepo}{\href{https://github.com/seohongpark/HILP}{this repository}\xspace}

\theoremstyle{plain}
\newtheorem{theorem}{Theorem}[section]

\newtheorem{corollary}[theorem]{Corollary}
\theoremstyle{definition}

\theoremstyle{remark}

\usepackage[textsize=tiny]{todonotes}

\icmltitlerunning{Foundation Policies with Hilbert Representations}

\begin{document}

\twocolumn[
\icmltitle{Foundation Policies with Hilbert Representations}

\icmlsetsymbol{equal}{*}

\begin{icmlauthorlist}
\icmlauthor{Seohong Park}{ucb}
\icmlauthor{Tobias Kreiman}{ucb}
\icmlauthor{Sergey Levine}{ucb}
\end{icmlauthorlist}

\icmlaffiliation{ucb}{University of California, Berkeley}

\icmlcorrespondingauthor{Seohong Park}{seohong@berkeley.edu}

\icmlkeywords{Machine Learning, ICML}

\vskip 0.3in
]

\printAffiliationsAndNotice{}  %

\cutsectionup
\begin{abstract}
\vspace{-2pt}
Unsupervised and self-supervised objectives, such as next token prediction, have enabled pre-training generalist models from large amounts of unlabeled data.
In reinforcement learning (RL), however, finding a truly general and scalable unsupervised pre-training objective for \emph{generalist policies} from offline data remains a major open question.
While a number of methods have been proposed to enable generic self-supervised RL, based on principles such as goal-conditioned RL, behavioral cloning, and unsupervised skill learning, such methods remain limited in terms of either the diversity of the discovered behaviors, the need for high-quality demonstration data, or the lack of a clear adaptation mechanism for downstream tasks.
In this work, we propose a novel unsupervised framework to pre-train generalist policies that capture diverse, optimal, long-horizon behaviors from unlabeled offline data such that they can be quickly adapted to any arbitrary new tasks in a zero-shot manner.
Our key insight is to learn a structured representation that preserves the temporal structure of the underlying environment, and then to span this learned latent space with directional movements, which enables various zero-shot policy ``prompting'' schemes for downstream tasks.
Through our experiments on simulated robotic locomotion and manipulation benchmarks, we show that our unsupervised policies can solve goal-conditioned and general RL tasks in a zero-shot fashion, even often outperforming prior methods designed specifically for each setting.
Our code and videos are available at \hilpwebsite.
\end{abstract}

\vspace{-22pt}
\section{Introduction}
\cutsectiondown

\begin{figure*}[t!]
    \centering
    \vspace{-3pt}
    \includegraphics[width=0.95\linewidth]{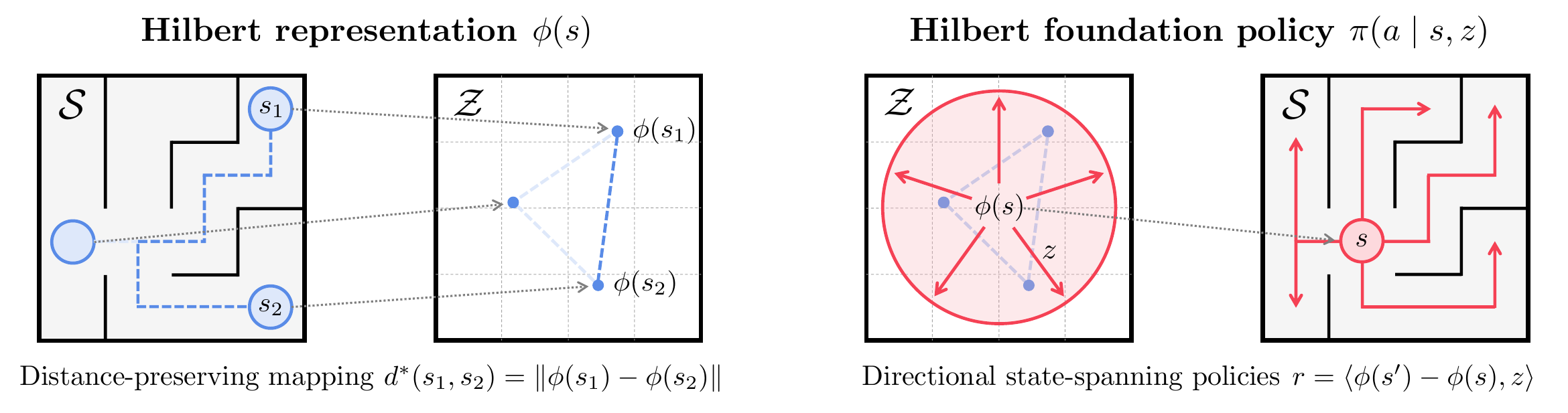}
    \vspace{-8pt}
    \caption{
    \footnotesize
    \textbf{Illustration of HILPs.}
    (\emph{left}) We first train a \emph{distance-preserving} mapping $\phi:\gS \to \gZ$
    that maps temporally similar states to spatially similar latent states ($d^*$ denotes the temporal distance).
    (\emph{right}) We then train a latent-conditioned policy $\pi(a \mid s, z)$, which we call a Hilbert foundation policy, that \emph{spans} that latent space with directional movements. %
    This policy captures diverse long-horizon behaviors from unlabeled data,
    which can be directly used to solve a variety of downstream tasks efficiently, even in a zero-shot manner.
    }
    \label{fig:illust}
    \vspace{-5pt}
\end{figure*}

Generalist models that can utilize large amounts of weakly labeled data provide an appealing recipe:
pre-train via self-supervised or unsupervised objectives on large and diverse datasets without ground truth labels,
and then adapt efficiently via prompting, few-shot learning, or fine-tuning to downstream tasks.
This strategy has proven to be extremely effective in settings where simple self-supervised objectives can be used to train on Internet-scale data~\citep{gpt3_brown2020,dalle2_ramesh2022},
leading to models that can quickly adapt to new tasks for pattern recognition~\citep{sam_kirillov2023},
question answering~\citep{instructgpt_ouyang2022},
and even diverse AI-assistant applications~\citep{codex_chen2021}.
Motivated by this observation,
a number of works have recently sought to propose self-supervised objectives to pre-train generalist \emph{policies}
for reinforcement learning (RL) and control~\citep{gato_reed2022,rtx_padalkar2024}.
We can broadly refer to the resulting models as \emph{foundation policies}:
general-purpose policies that can rapidly adapt to solve a variety of downstream tasks.

However, unlike natural language processing,
where next token prediction has become the standard pre-training objective~\citep{gpt3_brown2020},
finding the best \emph{policy pre-training objective} from data remains a major open question in RL.
Prior works have proposed several ways to pre-train generalist policies
based on diverse objectives, such as
behavioral cloning (BC)~\citep{opal_ajay2021,gato_reed2022,rtx_padalkar2024},
offline goal-conditioned RL (GCRL)~\citep{am_chebotar2021,contrastive_eysenbach2022,hiql_park2023},
and unsupervised skill discovery~\citep{vic_gregor2016,eigen_machado2017,diayn_eysenbach2019,metra_park2024}.
However, none of these objectives is ideal: behavioral cloning requires expert demonstrations,
which limits the availability of data,
goal-conditioned RL can only yield goal-reaching behaviors, %
and unsupervised skill discovery methods, though general and principled,
can present major challenges in terms of scalability, optimization, and offline learning.

In this work, we propose a general offline pre-training objective for foundation policies
that capture diverse, optimal ``long-horizon'' behaviors from unlabeled data
to facilitate downstream task learning.
Our main idea is to discover the \emph{temporal structure} of states through offline data,
and to represent this structure in such a way that
we can quickly and accurately obtain optimal policies
for any arbitrary new tasks from relatively concise ``prompts''
(\eg, a small number of states annotated with rewards, target goals, etc.).
We begin by learning a
geometric abstraction of the dataset,
where distances between representations of states
correspond to their long-horizon global relationships.
Specifically, we train a representation $\phi: \gS \to \gZ$
that maps the state space $\gS$ into a \emph{Hilbert} space $\gZ$
(\ie, a metric space with a well-defined \emph{inner product})
such that
\begin{align}
d^*(s, g) = \|\phi(s) - \phi(g)\|
\end{align}
holds for every $s, g \in \gS$,
where $d^*$ denotes the temporal distance (\ie, the minimum number of time steps needed for an optimal policy to transition between them).
Then, we train a latent-conditioned policy $\pi(a \mid s, z)$ that \emph{spans}
the learned latent space using offline RL,
with the following ``directional'' intrinsic reward based on the inner product:
\begin{align}
r(s, z, s') = \langle \phi(s') - \phi(s), z \rangle.
\end{align}
Intuitively, by learning to move in every possible \emph{direction} specified by a unit vector $z \in \gZ$,
the policy learns diverse long-horizon behaviors
that optimally \emph{span} the latent space as well as the state space (\Cref{fig:illust}).
The resulting multi-task policy $\pi(a \mid s, z)$ has a number of attractive properties.
\textbf{First}, it captures a variety of diverse behaviors, or \emph{skills}, from offline data.
These behaviors can be hierarchically combined or fine-tuned to solve downstream tasks efficiently.
\textbf{Second}, we can train this policy with offline RL (as opposed to BC), and thus can utilize suboptimal data,
unlike BC-based policy pre-training methods.
Moreover, the learned behaviors are provably optimal for solving goal-reaching tasks (under some assumptions),
which makes our method subsume \emph{goal-conditioned RL} as a special case,
while providing for much more diverse behaviors.
\textbf{Third}, thanks to our inner product parameterization,
this multi-task policy provides a very efficient way to adapt to any arbitrary reward function,
enabling \emph{zero-shot RL}. %
\textbf{Fourth},
this pre-training procedure yields a highly structured Hilbert representation $\phi$,
which enables efficient test-time \emph{planning} without training an additional model. %
Given the above versatility of our multi-task policy $\pi(a \mid s, z)$,
we call it a \textbf{Hilbert foundation policy} (\textbf{HILP}).

Our main contribution of this work is to introduce HILPs,
a new objective to pre-train diverse policies from offline data
that can be adapted efficiently to various downstream tasks.
Through our experiments,
we empirically demonstrate that HILPs capture diverse behaviors
that can be directly used to solve goal-conditioned RL and zero-shot RL without any additional training.
We also show that our single general HILP framework often outperforms previous offline policy pre-training methods
specifically designed for individual problem statements (\eg, zero-shot RL, goal-conditioned RL, and hierarchical RL)
on seven robotic locomotion and manipulation environments.

\cutsectionup
\section{Related Work}
\cutsectiondown

\textbf{Representation learning for sequential decision making.}
HILPs are based on a distance-preserving state representation $\phi$,
and are related to prior work in representation learning for RL and control.
Previous methods have proposed various representation learning objectives
based on visual feature learning~\citep{rrl_shah2021,pvr_parisi2022,mvp_xiao2022},
contrastive learning~\citep{tcn_sermanet2018,r3m_nair2022},
dynamics modeling~\citep{apv_seo2022,acstate_lamb2022,idm_brandfonbrener2023},
and goal-conditioned RL~\citep{vip_ma2023,icvf_ghosh2023}.
In particular, several previous methods~\citep{tcn_sermanet2018,r3m_nair2022,vip_ma2023}
employ the same $\ell^2$ parameterization as HILPs
to obtain temporal distance-based representations.
However,
unlike these prior works, which focus only on pre-training representations,
our focus is on unsupervised pre-training of diverse \emph{behaviors} (\ie, foundation \emph{policies}).
This enables solving downstream tasks in a \emph{zero-shot} manner
by simply ``prompting'' the foundation policy.

\textbf{Unsupervised policy pre-training.}
Prior works have proposed various unsupervised (\ie, task-agnostic) objectives
to pre-train diverse policies that can be used to accelerate downstream task learning.
Online unsupervised RL methods pre-train policies with
exploration~\citep{icm_pathak2017,disag_pathak2019,lexa_mendonca2021,murlb_rajeswar2023}
or skill discovery objectives~\citep{vic_gregor2016,diayn_eysenbach2019,dads_sharma2020,dceo_klissarov2023,metra_park2024}.
Unlike these works, we focus on the offline setting,
where we aim to learn diverse policies purely from an offline dataset of unlabeled trajectories.

For offline policy pre-training,
behavioral cloning~\citep{spirl_pertsch2020,opal_ajay2021,rtx_padalkar2024}
and trajectory modeling~\citep{dt_chen2021,tt_janner2021,gato_reed2022,maskdp_liu2022,mtm_wu2023}
approaches train foundation policies via supervised learning.
However, these supervised learning-based methods share a limitation
in that they assume demonstrations of high quality.
Among offline RL-based policy pre-training approaches, offline goal-conditioned RL methods train
goal-conditioned policies to reach any goal state from any other state~\citep{contrastive_eysenbach2022,gofar_ma2022,goat_yang2023,quasi_wang2023,hiql_park2023}.
These methods, however, only learn
goal-reaching behaviors and thus have limited behavioral diversity.
In contrast, our method subsumes goal-conditioned RL as a special case
while learning much more diverse behaviors,
which can be used to maximize arbitrary reward functions in a zero-shot manner.

Another line of work pre-trains multi-task policies with offline RL
based on successor features and other generalized value function designs~\citep{sr_dayan1993,sf_barreto2017,usfa_borsa2019,usf_ma2020,fb_touati2021,zs_touati2023,ramp_chen2023,uber_hu2023}.
Our work is closely related to these approaches as our inner product reward function resembles the linear structure in the successor feature framework.
Any successor feature or generalized value function approach operating as an unsupervised pre-training method
needs to make a key decision about which tasks to learn,
since any finite task representation needs to trade off some tasks for others.
Some prior methods make this decision simply based on random reward functions or random features~\citep{random_zheng2021,pvn_farebrother2023,ramp_chen2023,uber_hu2023},
while the others employ 
hand-crafted state features~\citep{sf_barreto2017,usfa_borsa2019},
off-the-shelf representation learning (\eg, autoencoders),
or low-rank approximation of optimal successor representations~\citep{zs_touati2023},
to specify and prioritize which tasks to capture.
In this work, we prioritize \emph{long-term temporal structure},
training state representation $\phi$ to capture the temporal distances between states
by geometrically abstracting the state space.
In our experiments, we show that this leads to significantly better performance and scalability
than prior successor feature- or generalized value function-based offline unsupervised RL methods.

Finally, our method is closely related to METRA~\citep{metra_park2024},
a recently proposed online unsupervised skill discovery method.
METRA also learns to span a temporal distance-based abstraction of the state space
based on a similar directional objective with online rollouts.
However, METRA cannot be directly applied to the offline setting as it assumes on-policy rollouts to train the representation $\phi$.
Unlike METRA, we decouple representation learning and policy learning to enable \emph{offline} policy pre-training from unlabeled data.

\cutsectionup
\section{Preliminaries and Problem Setting}
\cutsectiondown
\label{sec:prelim}

\textbf{Markov decision process (MDP).}
An MDP $\gM$ is defined as a tuple $(\gS, \gA, r, \mu, p)$,
where $\gS$ is the state space, $\gA$ is the action space,
$r: \gS \to \sR$ is the reward function, $\mu: \Delta(\gS)$ is the initial state distribution,
and $p: \gS \times \gA \to \Delta(\gS)$ is the transition dynamics kernel.
In this work, we assume a deterministic MDP unless otherwise stated,
following prior works in offline RL
and representation learning~\citep{vip_ma2023,icvf_ghosh2023,quasi_wang2023}.

\textbf{Hilbert space.}
A Hilbert space $\gZ$ is a complete vector space
equipped with an inner product $\langle x, y \rangle$,
the induced norm $\| x \| = \sqrt{\langle x, x \rangle}$,
and the induced metric $d(x, y) = \|x - y\|$ for $x, y \in \gZ$. 
\textbf{Intuitively}, a Hilbert space can roughly be thought of as a ``stricter'' version of a metric space,
where there exists an \emph{inner product} that is consistent with the metric.
For example, a Euclidean space with the $\ell^1$- or $\ell^\infty$-norm
is a metric space but not a Hilbert space,
whereas a Euclidean space with the $\ell^2$-norm is a Hilbert space,
as $\|x\|_2 = \sqrt{x^\top x}$ for $x \in \sR^D$.
In our experiments, we will mainly employ Euclidean spaces (with the $\ell^2$-norm) as Hilbert spaces,
but the theorems in the paper can be applied to any arbitrary real Hilbert space.

\textbf{Problem setting.}
We assume that we are given unlabeled trajectory data $\gD$,
which consists of state-action trajectories $\tau = (s_0, a_0, s_1, \dots, s_T)$.
We do not make any assumptions about the quality of these unlabeled trajectories:
they can be optimal for some unknown tasks, suboptimal, completely random, or even a mixture of these.

Our goal is to pre-train a versatile latent-conditioned policy $\pi(a \mid s, z)$,
where $z \in \gZ$ denotes a latent vector (which we call a \emph{task} or a \emph{skill}),
purely from the unlabeled offline data $\gD$, without online interactions.
For the evaluation of the pre-trained policy,
we consider three evaluation settings.
(\textbf{1}) \textbf{Zero-shot RL}\footnote{We use the term ``zero-shot RL'' following \citet{zs_touati2023}.}:
Given a reward function $r(s)$,
we aim to find the best latent vector $z$ that maximizes the reward function,
without additional training.
(\textbf{2}) \textbf{Offline goal-conditioned RL}:
Given a target goal $g \in \gS$,
we aim to find the best latent vector $z$ of the policy $\pi(a \mid s, z)$
that leads to the goal as quickly as possible, without additional training.
The goal is specified at test time.
(\textbf{3}) \textbf{Hierarchical RL}:
Given a reward function $r(s)$,
we train a high-level policy $\pi^h(z \mid s)$ that sequentially combines pre-trained skills
to maximize the reward function using offline RL.
In all three settings, we only allow online interaction with the environment during the final evaluation,
and assume that the state space and environment dynamics remain the same at evaluation time.

\cutsectionup
\section{Hilbert Foundation Policies (HILPs)}
\cutsectiondown

We now introduce our offline pre-training scheme for foundation policies
that capture diverse long-horizon behaviors from unlabeled data.
Our main strategy is to first learn a geometric state abstraction that preserves the temporal structure of the MDP (\Cref{sec:rep}),
and then to span the abstracted latent space with skills that correspond to directional movements in this space (\Cref{sec:policy}).

\cutsubsectionup
\subsection{Hilbert Representations}
\cutsubsectiondown
\label{sec:rep}

We begin by training a representation function $\phi: \gS \to \gZ$
that abstracts the state space into a
latent space $\gZ$.
We have two desiderata for $\phi$.
First, $\phi$ should map \emph{temporally} similar states to \emph{spatially} similar latent states,
so that it can abstract the dataset states while preserving their long-horizon global relationships.
Second, $\phi$ should be well-structured such that
it provides a way to train a versatile multi-task policy $\pi(a \mid s, z)$
that can be easily ``prompted'' to solve a variety of downstream tasks.

Based on these desiderata,
we set $\gZ$ to be a Hilbert space,
which not only provides a proper \emph{metric} to quantify the similarity between latent states,
but also provides an \emph{inner product} that enables several principled ways to prompt the policy,
which we will describe in \Cref{sec:hilp_usage}.
In the latent space $\gZ$,
our desiderata for the representation function $\phi$ can be formalized as follows:
\begin{align}
d^*(s, g) = \|\phi(s) - \phi(g)\|, \label{eq:temp_obj}
\end{align}
where $d^*$ denotes the optimal temporal distance from $s$ to $g$,
\ie, the minimum number of time steps to reach $g$ from $s$.
We refer to a representation $\phi$ that satisfies \Cref{eq:temp_obj} as a \textbf{Hilbert representation}.
Intuitively, $\phi$ is a distance-preserving embedding function (\ie, an \emph{isometry} to a Hilbert space),
where distances in the latent space correspond to the temporal distances in the original MDP.
This enables $\phi$ to abstract the state space while maintaining the global relationships between states.

To train $\phi$, we leverage the equivalence between temporal distances and optimal goal-conditioned value functions~\citep{gcrl_kaelbling1993,quasi_wang2023},
$V^*(s, g) = -d^*(s, g)$.
Here, $V^*(s, g)$ is the optimal goal-conditioned value function for the state $s$ and the goal $g$,
\ie, the maximum possible return (\ie, the sum of rewards) for the reward function given by
$r(s, g) = -\mathds{1}(s \neq g)$
and the episode termination condition given by $\mathds{1}(s=g)$.
Based on this connection to goal-conditioned RL,
we can train $\phi$ with any off-the-shelf offline goal-conditioned value learning algorithm~\citep{hiql_park2023,vip_ma2023,quasi_wang2023}
with the value function being parameterized as
\begin{align}
V(s, g) = -\|\phi(s) - \phi(g)\|. \label{eq:gcvf}
\end{align}

\textbf{Implementation.}
For practical implementation,
we set $\gZ$ to be the Euclidean space $\sR^D$ with the $\ell^2$-norm.
To train $V(s, g)$ from offline data,
we opt to employ the IQL-based~\citep{iql_kostrikov2022}
goal-conditioned value learning scheme introduced by \citet{hiql_park2023}.
This method minimizes the following temporal difference loss:
\begin{align}
\E_{s, s', g}[\ell^2_\tau (-\mathds{1}(s \neq g) + \gamma \bar V(s', g) - V(s, g))], \label{eq:obj_gcvf}
\end{align}
where $\gamma$ denotes a discount factor,
$\bar V$ denotes the target network~\citep{dqn_mnih2013},
and $\ell_\tau^2(x) = |\tau - \mathds{1}(x < 0)| x^2$ denotes the expectile loss~\citep{exp_newey1987},
an asymmetric $\ell^2$ loss that approximates the $\max$ operator in the Bellman backup~\citep{iql_kostrikov2022}.
States, next states, and goals (\ie, $(s, s', g)$) are sampled from the replay buffer with a hindsight relabeling strategy~\citep{her_andrychowicz2017,hiql_park2023},
and episodes terminate upon goal reaching
(see \Cref{sec:exp_detail} for details).
With the value function parameterization in \Cref{eq:gcvf},
our final objective for Hilbert representations $\phi$ becomes
\begin{align}
\E[\ell^2_\tau (-\mathds{1}(s \neq g) - \gamma \|\bar \phi(s') - \bar \phi(g)\| + \|\phi(s) - \phi(g)\|)], \label{eq:obj_phi}
\end{align} \\[-19pt]
where $\bar \phi$ denotes the target representation network.

\textbf{Remarks.}
There exist three potential limitations with \Cref{eq:obj_phi}.
First, the distance metric in $\gZ$ is symmetric,
whereas temporal distances might be asymmetric.
Second, even when the environment dynamics are symmetric,
there might not exist an exact isometry between the MDP and the Hilbert space~\citep{dist_indyk2017,dist_pitis2020}.
Third, we use a discount factor $\gamma$ in \Cref{eq:obj_phi}, but temporal distances are undiscounted.
In this regard, our objective might better be viewed as
finding the best discounted Hilbert \emph{approximation} of the MDP,
rather than learning an exact Hilbert abstraction.
While this might not be ideal in highly asymmetric environments,
we note that we will \emph{not} directly use the potentially erroneous parameterized value function $V(s, g)$ for policy learning.
We will instead only take the representation $\phi$,
defining a new reward function as well as a new value function
to pre-train an unsupervised latent-conditioned policy, as we will describe in \Cref{sec:policy}.
After all, our goal is to train a representation $\phi$ that captures the long-term temporal structure of the MDP,
and we empirically found that even approximate Hilbert representations lead to diverse useful behaviors that can be directly used to solve downstream tasks in our experiments (\Cref{sec:exp}). %
We refer to \Cref{sec:theory} for further discussions.

We also note that the parameterization in \Cref{eq:gcvf} has been employed in several prior works
in robotic representation learning~\citep{tcn_sermanet2018,r3m_nair2022,vip_ma2023}.
However, unlike these works, which mainly use $\phi$ only as a visual feature extractor,
we pre-train a foundation \emph{policy} with an intrinsic reward function based on an inner product involving $\phi$.

\begin{figure}[t!]
    \centering
    \includegraphics[width=0.8\linewidth]{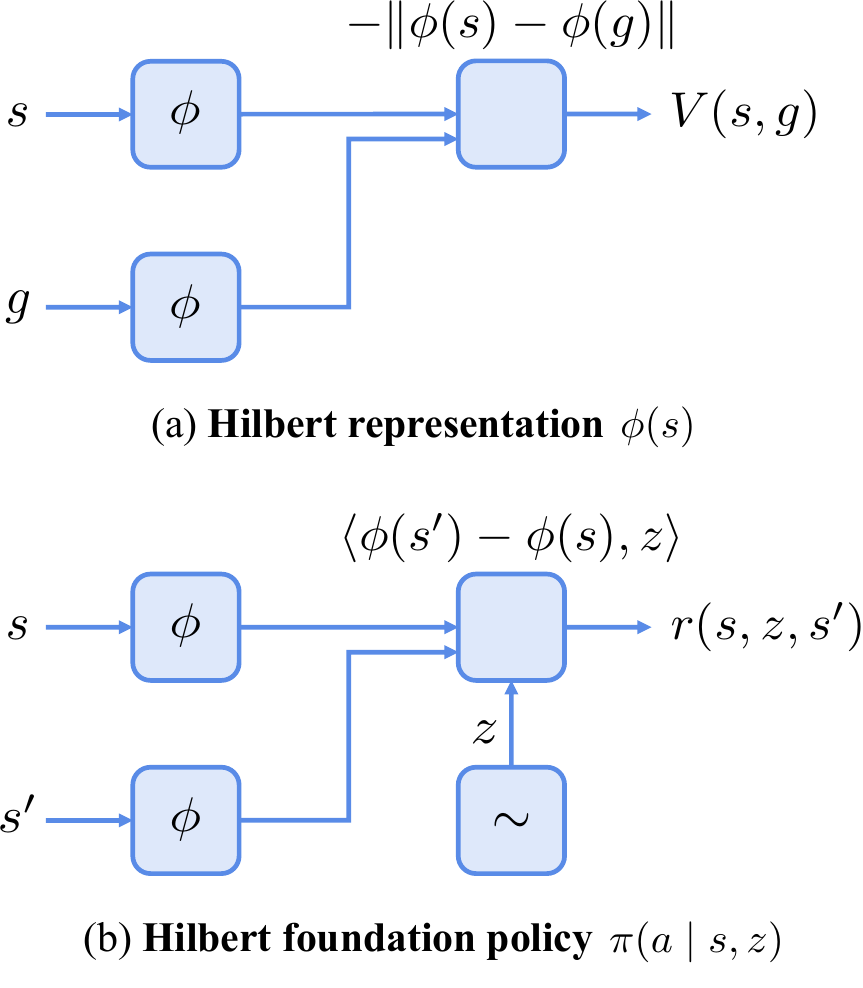}
    \vspace{-10pt}
    \caption{
    \footnotesize
    \textbf{Diagram of HILPs.}
    (\emph{a}) We train a Hilbert representation $\phi(s)$ using a goal-conditioned value learning objective
    with the value function parameterized as $V(s, g) = -\|\phi(s) - \phi(g)\|$.
    (\emph{b}) We train a Hilbert foundation policy $\pi(a \mid s, z)$
    using the intrinsic reward function $r(s, z, s')$
    defined as the inner product between $\phi(s') - \phi(s)$ and a randomly sampled unit vector $z$.
    }
    \label{fig:arch}
\end{figure}

\cutsubsectionup
\subsection{Unsupervised Policy Training}
\cutsubsectiondown
\label{sec:policy}

After obtaining a Hilbert representation $\phi$,
our next step is to learn diverse skills that \emph{span} the latent space $\gZ$ with offline RL.
Since a Hilbert space provides an inner product,
we can train a state-spanning
latent-conditioned policy $\pi(a \mid s, z)$ with the following inner product-based reward function:
\begin{align}
r(s, z, s') = \langle \phi(s') - \phi(s), z \rangle, \label{eq:hilp_reward}
\end{align}
where $z$ is sampled uniformly from the set of unit vectors in $\gZ$, $\{z \in \gZ: \|z\|=1\}$.

Since the latent-conditioned policy must maximize the reward in \Cref{eq:hilp_reward}
for all randomly sampled unit vectors $z$,
the optimal set of skills represented by $\pi(a \mid s, z)$
should be able to travel as far as possible in every possible latent space direction.
Consequently, we obtain a set of policies that optimally span the latent space,
capturing diverse behaviors from the unlabeled dataset $\gD$.
We call the resulting policy $\pi(a \mid s, z)$ a \textbf{Hilbert foundation policy} (\textbf{HILP}).
In \Cref{sec:hilp_usage}, we will discuss why HILPs are useful for solving downstream tasks in a variety of scenarios.

\textbf{Implementation.}
To train a HILP, we employ a standard off-the-shelf offline RL algorithm,
such as IQL~\citep{iql_kostrikov2022},
with the intrinsic reward defined as \Cref{eq:hilp_reward}.
Latent vectors $z$ are sampled from the uniform distribution over $\sS^{D-1} = \{z \in \sR^D: \|z\| = 1\}$.
In practice, we also consider another variant of the reward function defined as
$r(s, z, s') = \langle \phi(s) - \bar \phi, z \rangle$,
where $\bar \phi$ is defined as $\E_{s \sim \gD}[\phi(s)]$.
This is also a state-spanning directional reward function,
but the displacement is defined as the difference from the center
instead of the difference between two adjacent states.
We found this variant to perform better in the zero-shot RL setting in our experiments (\Cref{sec:exp_zs_rl}).
We illustrate the architecture of HILPs in \Cref{fig:arch},
summarize the full training procedure in \Cref{alg:algo},
and refer to \Cref{sec:exp_detail} for the full experimental details.

\cutsectionup
\section{Why are HILPs useful?}
\cutsectiondown
\label{sec:hilp_usage}

HILPs, in combination with structured representations $\phi(s)$,
provide a number of ways to solve downstream tasks efficiently, often even in a zero-shot fashion.
In this section, we describe three test-time ``prompting'' strategies for HILPs
for zero-shot RL (\Cref{sec:zs_rl}), offline goal-conditioned RL (\Cref{sec:zs_gcrl}),
and test-time planning (\Cref{sec:planning}).

\cutsubsectionup
\subsection{Zero-Shot RL}
\cutsubsectiondown
\label{sec:zs_rl}

First, HILPs can be used to quickly adapt to any arbitrary reward functions at test time.
Our key observation is that
the operand in our inner product reward function (\Cref{eq:hilp_reward})
$\tilde \phi(s, a, s') := \phi(s') - \phi(s)$
can be viewed as a \emph{cumulant}~\citep{rl_sutton2005}
in the successor feature framework~\citep{sr_dayan1993,sf_barreto2017}.
This connection to successor features enables \emph{zero-shot RL}~\citep{zs_touati2023}:
given an arbitrary reward function $r(s, a, s')$ at test time,
we can find the latent vector $z$ for the policy $\pi(a \mid s, z)$
that best solves the task via simple linear regression,
without any additional training.
Specifically, the optimal (unnormalized) latent vector $z^*$ for the reward function $r(s, a, s')$ is given as
\begin{align}
z^* = \argmin_{z \in \gZ} \E_\gD \left[\left(r(s, a, s') - \langle \tilde \phi(s, a, s'), z \rangle \right)^2\right],
\end{align}
where $(s, a, s')$ tuples are sampled from the unlabeled dataset $\gD$.
If $\gZ$ is a Euclidean space, we have the closed-form solution
\mbox{$z^* = \E_\gD [ \tilde \phi \tilde \phi^\top ]^{-1}\E_\gD [ r(s, a, s') \tilde \phi ]$},
where $\tilde \phi$ denotes $\tilde \phi(s, a, s')$.

\begin{algorithm}[t!]
    \small
    \caption{Hilbert Foundation Policies (HILPs)}
    \begin{algorithmic}[1]
        \begingroup
        \STATE Initialize Hilbert representation $\phi(s)$, policy $\pi(a \mid s, z)$
        \WHILE{not converged}
            \STATE Sample $(s, s', g) \sim \gD$
            \STATE Train $\phi(s)$ by minimizing \Cref{eq:obj_phi}
        \ENDWHILE
        \WHILE{not converged}
            \STATE Sample $(s, a, s') \sim \gD$
            \STATE Sample $z \sim \sS^{D-1}$
            \STATE Compute intrinsic reward $r(s, z, s')$
            \STATE Train $\pi(a \mid s, z)$ with any off-the-shelf offline RL algorithm
        \ENDWHILE
        \endgroup
    \end{algorithmic}
    \label{alg:algo}
\end{algorithm}

In practice,
we sample a small number of $(s, a, s')$ tuples from the dataset,
compute the optimal latent $z^*$ with respect to the test-time reward function using the above formula,
and execute the corresponding policy $\pi(a \mid s, z^*)$ to perform zero-shot RL.

\cutsubsectionup
\subsection{Offline Goal-Conditioned RL}
\cutsubsectiondown
\label{sec:zs_gcrl}

For goal-reaching tasks,
HILPs provide an even simpler way to solve \emph{goal-conditioned RL} in a zero-shot manner,
where the aim is to reach a goal state $g \in \gS$ within the minimum number of steps.
Intuitively, a Hilbert representation $\phi$ is learned in a way that the distance between $\phi(s)$ and $\phi(g)$ in the latent space
corresponds to the optimal temporal distance $d^*(s, g)$.
Hence, to reach the goal $g$, the agent just needs to move in the latent direction of $\phi(g) - \phi(s)$
so that it can monotonically decrease the distance toward $\phi(g)$ in the Hilbert space,
which in turn decreases the temporal distances toward $g$ in the original MDP.
Since the HILP $\pi(a \mid s, z)$ is pre-trained to move in every possible direction,
we can just set $z$ to
\begin{align}
    z^* := \frac{\phi(g) - \phi(s)}{\|\phi(g) - \phi(s)\|}. \label{eq:gc_z}
\end{align}

\begin{figure}[t!]
    \centering
    \vspace{-3pt}
    \includegraphics[width=0.8\linewidth]{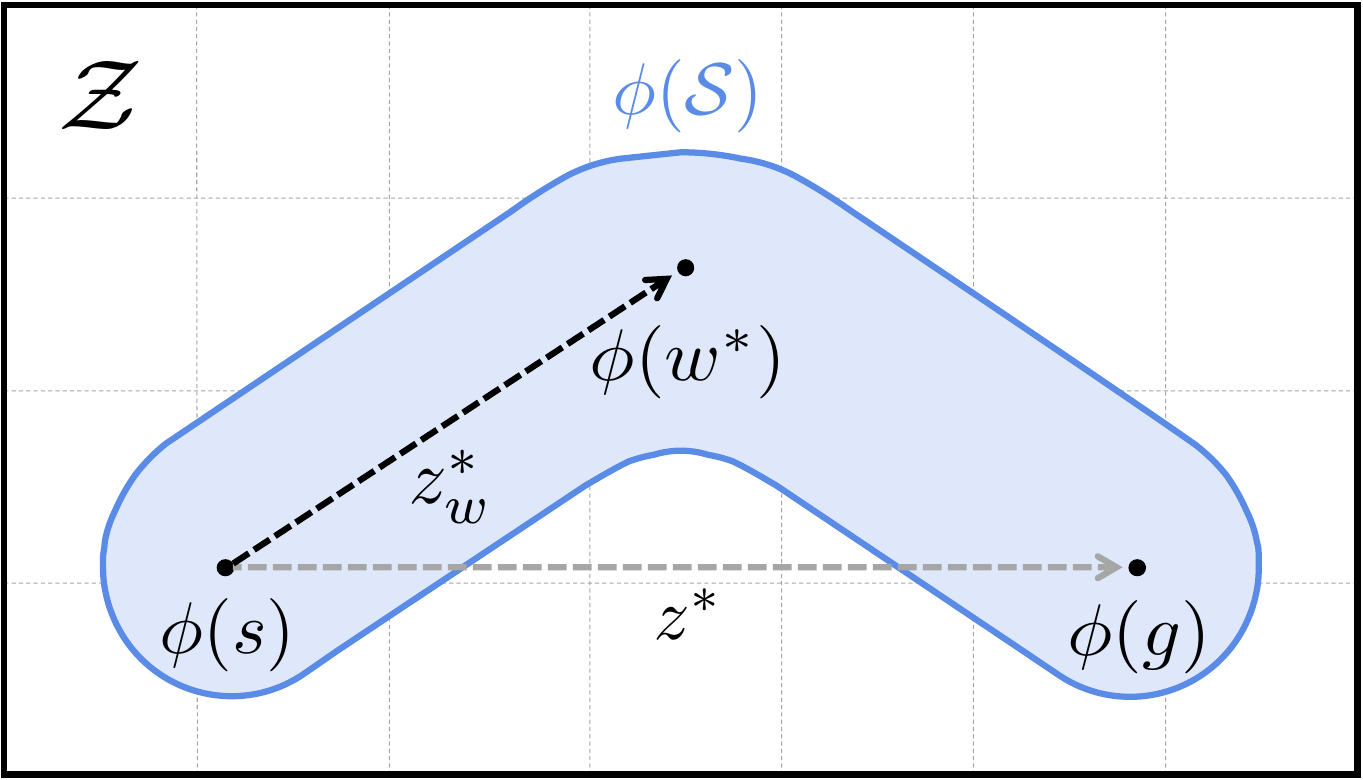}
    \vspace{-5pt}
    \caption{
    \footnotesize
    \textbf{Test-time midpoint planning.}
    In the presence of embedding errors,
    the direction toward the midpoint subgoal $w^*$ can be more accurate than the direction toward the goal $g$.
    }
    \label{fig:plan_illust}
    \vspace{-8pt}
\end{figure}

We theoretically prove this intuition as follows:
\begin{theorem}[Directional movements in the latent space are optimal for goal reaching]
\label{thm:gcrl}
If embedding errors are bounded as \mbox{$\sup_{s, g \in \gS} |d^*(s, g) - \|\phi(s) - \phi(g)\|| \leq \eps_e$},
directional movement errors are bounded as \mbox{$\sup_{s, g \in \gS} \|z'^*(s, g) - \hat z'(s, g)\| \leq \eps_d$},
and $\, 4\eps_e + \eps_d < 1$,
then $\pi(a \mid s, z^*)$ is an optimal goal-reaching policy.
\end{theorem}
The formal definitions of $z'^*(s, g)$ and $z'(s, g)$ and the proof are provided in \Cref{sec:theory}.
Intuitively, \Cref{thm:gcrl} states that if the embedding errors of $\phi$ are small enough,
directional movements in the latent space are optimal for solving goal-reaching tasks.
We refer to \Cref{sec:theory} for further discussion including limitations.

\cutsubsectionup
\subsection{Test-Time Planning}
\cutsubsectiondown
\label{sec:planning}

Another benefit of our HILP framework is that
it naturally enables efficient \emph{test-time planning} for goal-conditioned RL
based on the structured state representation $\phi(s)$.
While \Cref{sec:zs_gcrl} introduces a simple method to query the policy $\pi(a \mid s, z)$ to reach a goal $g$ (\Cref{eq:gc_z}),
this might not be perfect in practice due to potential embedding errors in $\phi$
and thus may potentially lead to suboptimal goal-reaching performance, as stated in \Cref{thm:gcrl}.
In particular, when the image of the mapping $\phi$ is distorted (\Cref{fig:plan_illust}),
the straight line from the current latent state $\phi(s)$ to the latent goal $\phi(g)$ in $\gZ$
might not represent a feasible path between $s$ and $g$ in the MDP, potentially making the agent struggle to reach the goal.

\begin{algorithm}[t!]
    \small
    \caption{Test-time planning with HILPs}
    \begin{algorithmic}[1]
        \STATE \textbf{Input}: unlabeled offline dataset $\gD$, Hilbert representation $\phi(s)$, HILP $\pi(a \mid s, z)$, goal $g$
        \STATE Sample $w_1, w_2, \dots, w_N \sim \gD$
        \STATE Pre-compute $\phi(w_1), \phi(w_2), \dots, \phi(w_N)$
        \STATE Observe $s_0 \sim \mu(s_0)$
        \FOR{$t \gets 0$ to $T-1$}
            \STATE $s, u \gets s_t, g$
            \FOR{$i \gets 1$ to (\# recursions)}
                \STATE $w^* \gets \argmin_{w \in \{w_1 \dots, w_N\}} \max(\|\phi(s) - \phi(w)\|,$ \newline
                    \hspace*{6pt} $\|\phi(w) - \phi(u)\|)$
                \STATE $u \gets w^*$
            \ENDFOR
            \STATE Compute $z_w^*$ with \Cref{eq:gc_w}
            \STATE Sample $a_t \sim \pi(a_t \mid s_t, z_w^*)$
            \STATE Observe $s_{t+1} \sim p(s_{t+1} \mid s_t, a_t)$
        \ENDFOR
    \end{algorithmic}
    \label{alg:plan}
\end{algorithm}

To resolve this issue,
we introduce a way to further \emph{refine} the latent vector $z$ to mitigate such approximation errors
by finding the optimal \emph{feasible subgoal} between $s$ and $g$.
Specifically, we aim to find a midpoint state that is equidistant from both $s$ and $g$ and still makes progress towards the goal.
This can be formalized as the following minimax objective~\citep{ris_chanesane2021}:
\begin{align}
w^* := &\argmin_{w \in \gS} \left[ \max(d^*(s, w), d^*(w, g)) \right] \label{eq:plan} \\
\approx &\argmin_{w \in \gS} \left[ \max(\|\phi(s) - \phi(w)\|, \|\phi(w) - \phi(g)\|) \right]. \nonumber
\end{align}

After finding $w^*$, we refine the latent vector $z$ for the policy $\pi(a \mid s, z)$ as follows:
\begin{align}
    z_w^* = \frac{\phi(w^*) - \phi(s)}{\|\phi(w^*) - \phi(s)\|}. \label{eq:gc_w}
\end{align}
As the distance between $s$ and $w^*$ is smaller than that between $s$ and $g$,
$z_w^*$ is likely more accurate than $z^*$ in \Cref{eq:gc_z}.
We illustrate this refinement procedure in \Cref{fig:plan_illust}.

\begin{figure*}[t!]
    \centering
    \includegraphics[width=1\linewidth]{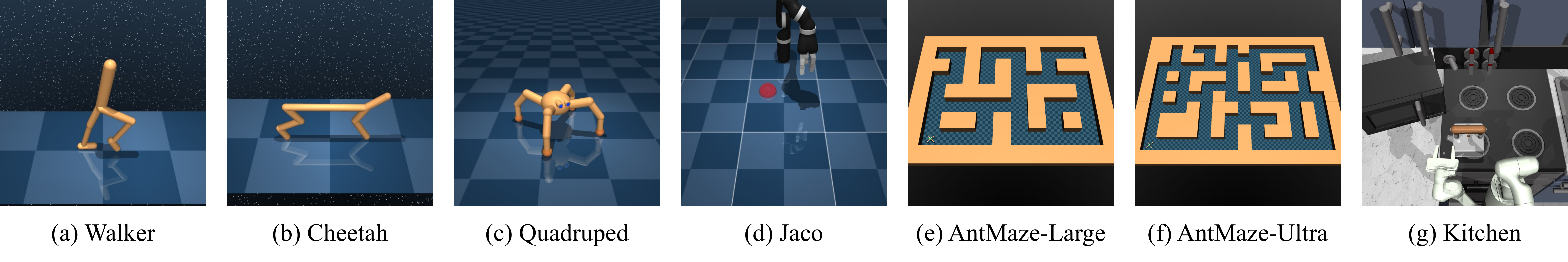}
    \vspace{-20pt}
    \caption{
    \footnotesize
    \textbf{Environments.}
    We evaluate HILPs on seven robotic locomotion and manipulation environments.
    }
    \vspace{-5pt}
    \label{fig:envs}
\end{figure*}

Thanks to our structured Hilbert representation $\phi$, \Cref{eq:plan} can be very efficiently computed in practice.
At test time, we first randomly sample a small number ($N$) of states from the dataset $\gD$,
and pre-compute their Hilbert representations.
Then, at every time step, we approximate $w^*$ by finding the $\argmin$ of \Cref{eq:plan} over the $N$ samples
using the pre-computed representations.
Since we can approximate $d^*$ by measuring the distances between representations,
this procedure does not require any additional neural network queries.

To further refine the subgoal, we can recursively apply this midpoint planning procedure more than once.
Namely,
we can first find the midpoint $w^*$ between $s$ and $g$ via \Cref{eq:plan},
and then find the midpoint between $s$ and $w^*$ (\ie, the $1/4$-point between $s$ and $g$) with the same minimax objective,
and so on.
Again, thanks to our Hilbert representation,
this recursive planning can also be done only with elementary algebraic operations,
without additionally querying neural networks.
In our experiments,
we empirically found that iterative refinements do improve performance on long-horizon tasks.
We summarize the full test-time planning procedure in \Cref{alg:plan} in \Cref{sec:exp_detail}.

\cutsectionup
\section{Experiments}
\cutsectiondown
\label{sec:exp}

In our experiments, we empirically evaluate the performance of HILPs
on various types of downstream tasks.
In particular, we consider three different experimental settings:
zero-shot RL, offline goal-conditioned RL, and hierarchical RL
(see \Cref{sec:prelim} for the problem statements),
in which we aim to answer the following questions:
(1) Can HILPs be prompted to solve goal-conditioned and reward-based tasks in a zero-shot manner?
(2) Can a single general HILP framework outperform prior state-of-the-art approaches specialized in each setting?
(3) Does test-time planning improve performance?
In our experiments, we use $8$ random seeds unless otherwise stated,
and report $95\%$ bootstrap confidence intervals in plots
and standard deviations in tables.
Our code is available at \hilpanonrepo.

\cutsubsectionup
\subsection{Zero-Shot RL}
\cutsubsectiondown
\label{sec:exp_zs_rl}

We first evaluate HILPs in the \emph{zero-shot RL} setting,
where the aim is to maximize an arbitrary reward function given at test time without additional training.
For benchmarks,
following \citet{zs_touati2023},
we use the Unsupervised RL Benchmark~\citep{urlb_laskin2021}
and ExORL datasets~\citep{exorl_yarats2022},
which consist of trajectories collected by various unsupervised RL agents.
We consider four environments (Walker, Cheetah, Quadruped, and Jaco) (\Cref{fig:envs})
and four datasets collected by APS~\citep{aps_liu2021}, APT~\citep{apt_liu2021},
Proto~\citep{protorl_yarats2021}, and RND~\citep{rnd_burda2019} in each environment.
These datasets correspond to the four most performant unsupervised RL algorithms in Figure 2 in the work by \citet{exorl_yarats2022}.
In addition to these original state-based environments,
we also employ their \emph{pixel-based} variants with $64 \times 64 \times 3$-sized observation spaces
to evaluate the scalability of the methods to complex observations.
With these environments and datasets,
we train HILPs and eight previous zero-shot RL methods in three different categories:
(1) forward-backward (FB) representations~\citep{zs_touati2023}, 
a state-of-the-art zero-shot RL method,
(2) successor features (SFs)
with six different feature learners (forward dynamics model (FDM), Laplacian (Lap)~\citep{lap_wu2019}, autoencoder (AE),
inverse dynamics model (IDM), random features, and contrastive learning (CL)),
which includes the best feature learner (Laplacian) in \citet{zs_touati2023},
and (3) goal-conditioned RL (GC-TD3).
We use TD3~\citep{td3_fujimoto2018} as the base offline RL algorithm to train these methods,
as it is known to perform best in ExORL~\citep{exorl_yarats2022}.
After finishing unsupervised policy training,
we evaluate these methods on four test tasks in each environment
(\eg, Flip, Run, Stand, and Walk for Walker)
with their own zero-shot adaption strategies.
For HILPs, we use the zero-shot prompting scheme introduced in \Cref{sec:zs_rl}.
In the Jaco domain, we additionally consider the goal-conditioned prompting scheme (\Cref{sec:zs_gcrl}) for HILPs (HILP-G),
since it is a goal-oriented environment
(however, we do \emph{not} use this variant for overall performance aggregation).
For FB and SF methods, we use the zero-shot schemes introduced by \citet{zs_touati2023} based on
reward-weighted expectation or linear regression.
For the goal-conditioned RL baseline (GC-TD3),
we find the state with the highest reward value from the offline dataset and use it as the goal.

\begin{figure}[t!]
    \centering
    \vspace{-4pt}
    \begin{subfigure}[t!]{0.88\linewidth}
        \centering
        \includegraphics[width=\textwidth]{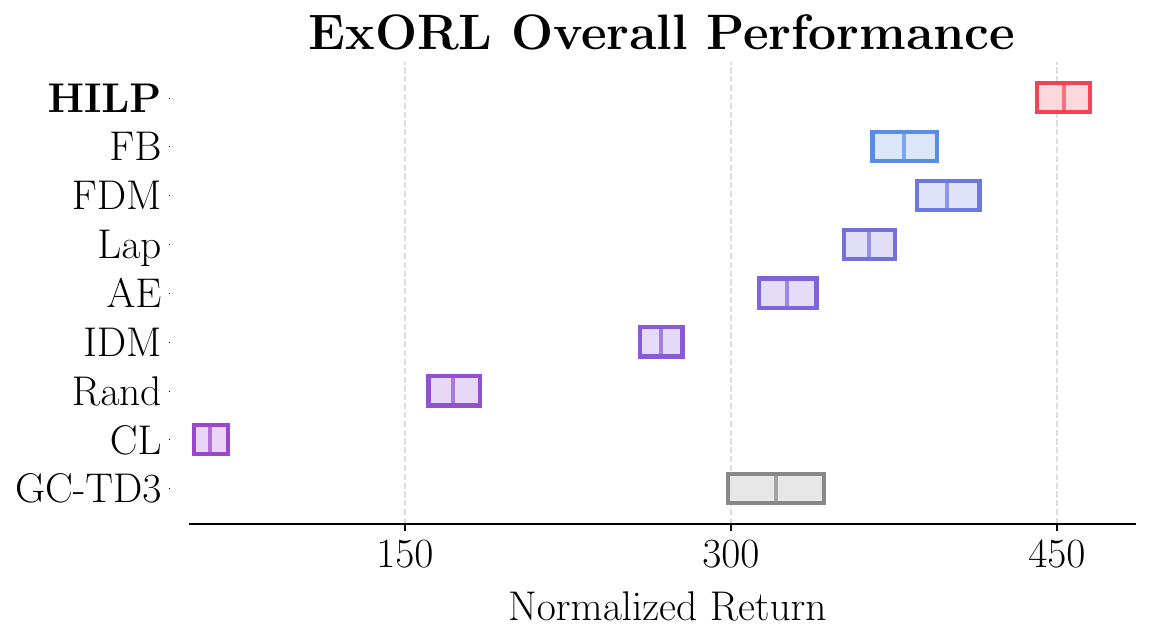}
    \end{subfigure}
    \begin{subfigure}[t!]{0.94\linewidth}
        \centering
        \includegraphics[width=\textwidth]{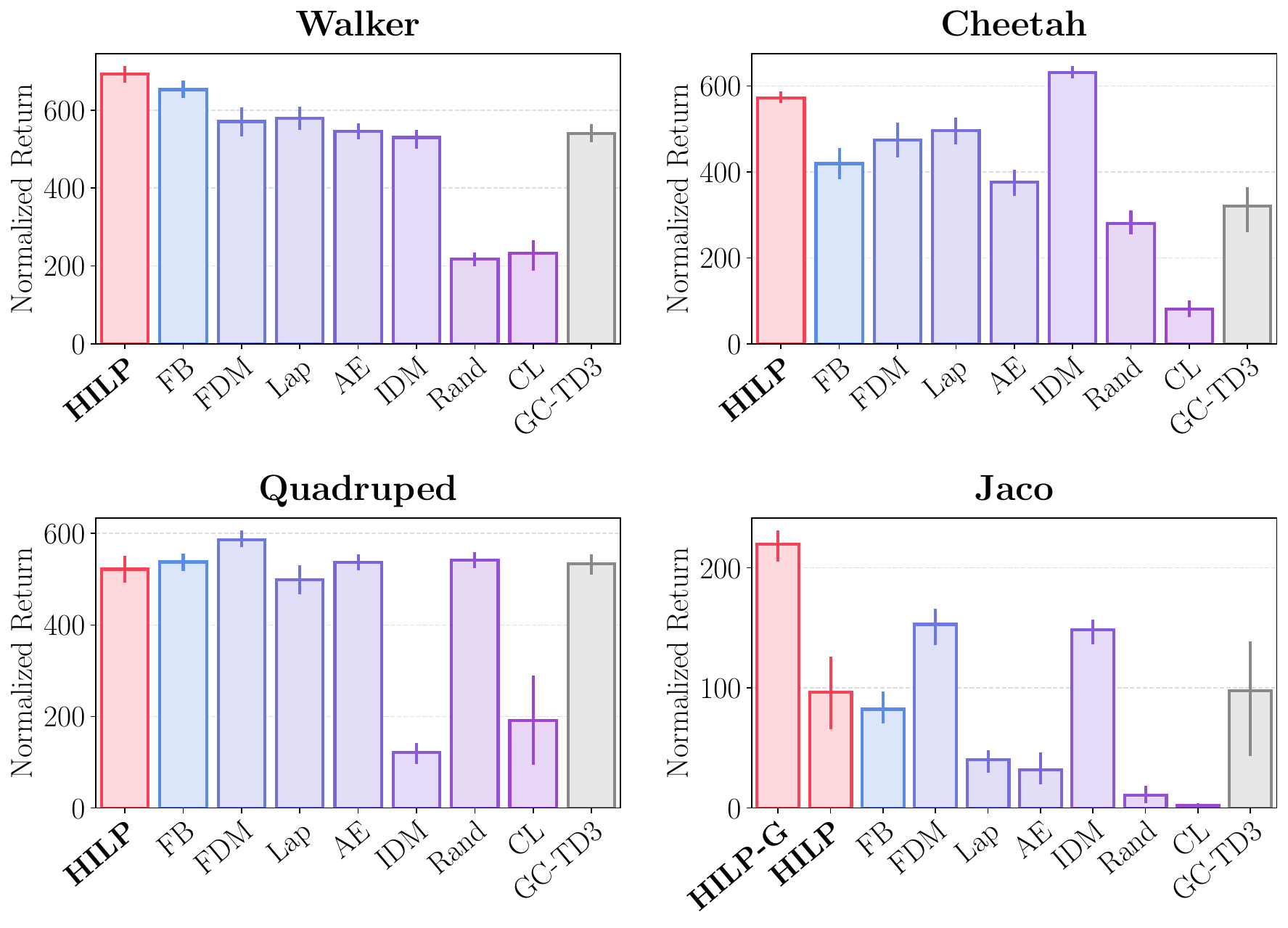}
    \end{subfigure}
    \vspace{-5pt}
    \caption{
    \textbf{Zero-shot RL performance.}
    HILP achieves the best zero-shot RL performance in the ExORL benchmark,
    outperforming previous state-of-the-art approaches.
    The overall results are aggregated over $4$ environments, $4$ tasks, $4$ datasets, and $4$ seeds (\ie, $256$ values in total).
    }
    \vspace{-13pt}
    \label{fig:zs}
\end{figure}

\Cref{fig:zs} shows the overall and per-environment comparison results,
where we use the interquartile mean (IQM) metric for overall aggregation,
following the suggestion by \citet{rliable_agarwal2021}.
The results suggest that HILPs achieve the best overall zero-shot RL performance among the nine methods,
while achieving the best or near-best scores in each environment.
Notably, HILPs achieve significantly better performance than GC-TD3,
which indicates that HILPs capture more diverse behaviors than ordinary goal-reaching policies,
and these diverse behaviors can be directly used to maximize a variety of reward functions at test time.
\Cref{fig:zs_pixel} shows the overall comparison results on pixel-based ExORL environments.
Due to high computational costs, we compare HILPs with the two most performant baselines in \Cref{fig:zs} (FB and FDM).
The results indicate that HILPs achieve the best IQM in pixel-based environments as well,
outperforming the previous best approaches.
We refer to \Cref{sec:add_results} for the full results.

\cutsubsectionup
\subsection{Offline Goal-Conditioned RL}
\cutsubsectiondown
\label{sec:exp_zs_gcrl}
Next, we evaluate HILPs on \emph{goal-reaching} tasks.
For benchmarks,
we consider the goal-conditioned variants of AntMaze and Kitchen tasks (\Cref{fig:envs})
from the D4RL suite~\citep{d4rl_fu2020,hiql_park2023}.
In the AntMaze environment,
we are given a dataset consisting of $1000$ trajectories of a quadrupedal Ant agent navigating through a maze.
We employ the two most challenging datasets with the largest maze (``antmaze-large-\{diverse, play\}'')
from the D4RL benchmark,
and two even larger settings (``antmaze-ultra-\{diverse, play\}'') introduced by \citet{tap_jiang2023}
to further challenge the long-horizon reasoning ability of the agent.
In the Kitchen manipulation environment~\citep{ril_gupta2019},
we are given a dataset consisting of trajectories of a robotic arm manipulating different kitchen objects
(\eg, a kettle, a microwave, cabinets, etc.) in various orders.
We employ two datasets (``kitchen-\{partial, mixed\}'') from the D4RL benchmark.
Additionally, we use \emph{pixel-based} variants of these datasets (``visual-kitchen-\{partial, mixed\}'')
to evaluate the visual reasoning capacity of the agent,
where the agent must learn to manipulate the robot arm purely from $64 \times 64 \times 3$ camera images.

\begin{figure}[t!]
    \centering
    \vspace{-4pt}
    \includegraphics[width=0.8\linewidth]{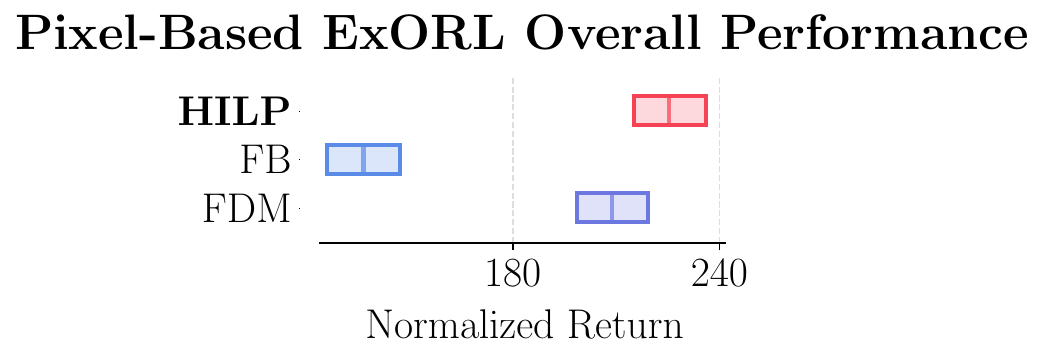}
    \vspace{-5pt}
    \caption{
    \textbf{Pixel-based zero-shot RL performance.}
    HILP exhibits the best performance in the \emph{pixel-based} ExORL benchmark as well,
    outperforming the two most performant baselines among FB and SF methods.
    The results are aggregated over $4$ environments, $4$ tasks, $4$ datasets, and $4$ seeds (\ie, $256$ values in total).
    }
    \vspace{-13pt}
    \label{fig:zs_pixel}
\end{figure}

In these environments, we compare HILPs against four general unsupervised policy pre-training methods
as well as three methods that are specifically designed to solve goal-conditioned RL.
For the former group, we consider the three best zero-shot RL methods from the previous section (FB, FDM, Lap)
and random successor features (Rand).
For goal-conditioned approaches, we consider
goal-conditioned behavioral cloning (GC-BC)~\citep{goalgail_ding2019,gcsl_ghosh2021},
goal-conditioned IQL (GC-IQL)~\citep{iql_kostrikov2022,hiql_park2023},
and goal-conditioned CQL (GC-CQL)~\citep{cql_kumar2020}.
To adapt HILPs to goal-conditioned tasks,
we employ the zero-shot prompting scheme (\Cref{sec:zs_gcrl})
as well as the test-time midpoint planning scheme with three recursions (\Cref{sec:planning}, ``HILP-Plan'').
For FB, we use the backward representation of the goal state to obtain the latent vector,
and for SF methods,
we perform linear regression with respect to the original reward function (\ie, they use privileged reward information),
as there is no clear way to adapt them to goal-conditioned tasks.

\begin{table*}[t!]
    \vspace{-5pt}
    \caption{
    \footnotesize
    \textbf{Offline goal-conditioned RL performance ($\mathbf{8}$ seeds).}
    HILP achieves the best performance among general offline unsupervised policy learning methods,
    while being comparable to methods that are specifically designed to solve offline goal-conditioned RL.
    With our efficient test-time planning procedure based on Hilbert representations, HILPs often even outperform offline goal-conditioned RL methods.
    }
    \vspace{5pt}
    \label{table:gcrl}
    \centering
    \scalebox{0.83}
    {
    \setlength{\tabcolsep}{5pt}
    \begin{tabular}{lrrrrrrrrr}
    \toprule
    \multicolumn{1}{c}{} & \multicolumn{3}{c}{GCRL-dedicated methods} & \multicolumn{6}{c}{General unsupervised policy learning methods}\\
    \cmidrule(lr){2-4} \cmidrule(lr){5-10}
    Dataset & GC-BC & GC-IQL & GC-CQL & FB & FDM & Lap & Rand & \textbf{HILP} (\textbf{ours}) & \textbf{HILP-Plan} (\textbf{ours}) \\
    \midrule
antmaze-large-diverse & $15.0$ {\tiny $\pm 9.3$} & $56.0$ {\tiny $\pm 6.0$} & $36.2$ {\tiny $\pm 19.0$} & $0.0$ {\tiny $\pm 0.0$} & $0.0$ {\tiny $\pm 0.0$} & $0.0$ {\tiny $\pm 0.0$} & $0.0$ {\tiny $\pm 0.0$} & $46.0$ {\tiny $\pm 12.7$} & $\mathbf{64.5}$ {\tiny $\pm 10.2$} \\
antmaze-large-play & $12.0$ {\tiny $\pm 5.9$} & $\mathbf{56.0}$ {\tiny $\pm 25.7$} & $32.0$ {\tiny $\pm 25.8$} & $0.0$ {\tiny $\pm 0.0$} & $0.0$ {\tiny $\pm 0.0$} & $0.0$ {\tiny $\pm 0.0$} & $0.0$ {\tiny $\pm 0.0$} & $49.0$ {\tiny $\pm 8.8$} & $\mathbf{58.8}$ {\tiny $\pm 11.2$} \\
antmaze-ultra-diverse & $30.5$ {\tiny $\pm 10.1$} & $40.8$ {\tiny $\pm 11.1$} & $14.2$ {\tiny $\pm 13.5$} & $0.0$ {\tiny $\pm 0.0$} & $0.0$ {\tiny $\pm 0.0$} & $0.0$ {\tiny $\pm 0.0$} & $0.0$ {\tiny $\pm 0.0$} & $21.2$ {\tiny $\pm 11.2$} & $\mathbf{59.2}$ {\tiny $\pm 12.7$} \\
antmaze-ultra-play & $26.5$ {\tiny $\pm 6.2$} & $41.8$ {\tiny $\pm 9.0$} & $16.5$ {\tiny $\pm 14.3$} & $0.0$ {\tiny $\pm 0.0$} & $0.0$ {\tiny $\pm 0.0$} & $0.0$ {\tiny $\pm 0.0$} & $0.0$ {\tiny $\pm 0.0$} & $22.2$ {\tiny $\pm 11.4$} & $\mathbf{50.8}$ {\tiny $\pm 9.6$} \\
kitchen-partial & $52.7$ {\tiny $\pm 11.0$} & $56.4$ {\tiny $\pm 8.4$} & $31.2$ {\tiny $\pm 16.6$} & $0.2$ {\tiny $\pm 0.7$} & $39.2$ {\tiny $\pm 4.2$} & $41.2$ {\tiny $\pm 9.3$} & $44.6$ {\tiny $\pm 8.9$} & $\mathbf{63.9}$ {\tiny $\pm 5.7$} & $59.7$ {\tiny $\pm 5.1$} \\
kitchen-mixed & $\mathbf{58.8}$ {\tiny $\pm 8.0$} & $\mathbf{59.5}$ {\tiny $\pm 3.8$} & $15.7$ {\tiny $\pm 17.6$} & $0.7$ {\tiny $\pm 1.9$} & $42.9$ {\tiny $\pm 7.1$} & $40.9$ {\tiny $\pm 3.0$} & $46.6$ {\tiny $\pm 2.3$} & $55.5$ {\tiny $\pm 9.5$} & $51.9$ {\tiny $\pm 8.3$} \\
visual-kitchen-partial & $\mathbf{63.2}$ {\tiny $\pm 3.7$} & $\mathbf{63.6}$ {\tiny $\pm 4.2$} & - & - & - & $43.3$ {\tiny $\pm 2.2$} & $39.7$ {\tiny $\pm 3.3$} & $59.9$ {\tiny $\pm 4.0$} & $52.1$ {\tiny $\pm 5.3$} \\
visual-kitchen-mixed & $\mathbf{57.5}$ {\tiny $\pm 4.2$} & $52.9$ {\tiny $\pm 4.7$} & - & - & - & $50.9$ {\tiny $\pm 3.6$} & $49.6$ {\tiny $\pm 5.7$} & $\mathbf{55.9}$ {\tiny $\pm 9.7$} & $54.1$ {\tiny $\pm 13.4$} \\
    \bottomrule
    \end{tabular}
    }
    \vspace{-5pt}
\end{table*}

\Cref{table:gcrl} shows the results,
which suggest HILPs can solve challenging long-horizon goal-conditioned tasks in a zero-shot manner,
and significantly outperform previous general unsupervised policy learning methods (FB and three SF methods).
This is likely because our method prioritizes learning long-horizon behaviors that span the state space,
unlike previous successor features or forward-backward methods.
\Cref{table:gcrl} also shows that
HILPs with test-time planning often even outperform GCRL-dedicated methods (\eg, GC-IQL).
We note that this planning procedure can be done only with elementary algebraic operations
based on pre-computed representations (\Cref{sec:planning}),
thanks to our structured Hilbert representations.
To further study the effect of test-time planning,
we compare the performances of HILPs with different numbers ($0$-$3$) of midpoint recursions.
We report results in \Cref{fig:gcrl_plan},
which shows that iterative refinement of prompts via our test-time planning approach improves performance consistently.
We refer to \Cref{sec:add_results} for further analyses,
including an \textbf{ablation study}, \textbf{embedding error analysis}, and \textbf{latent space visualization}.

\cutsubsectionup
\subsection{Hierarchical RL}
\cutsubsectiondown
\label{sec:exp_hrl}

Finally, to evaluate the effectiveness of skills learned by HILPs compared to previous BC-based skill learning methods,
we compare HILPs with OPAL~\citep{opal_ajay2021},
a previous offline skill extraction method based on a trajectory variational autoencoder (VAE).
For benchmarks, we use the AntMaze and Kitchen datasets from D4RL~\citep{d4rl_fu2020}.
On top of skills learned by HILPs or OPAL in these environments,
we train a high-level skill policy $\pi^h(z \mid s)$ that uses learned skills as high-level actions
with IQL~\citep{iql_kostrikov2022}.
We use the same IQL-based high-level policy training scheme to ensure a fair comparison.
\Cref{table:hrl} shows the results,
which suggest that HILPs achieve better performances than OPAL,
especially in the most challenging AntMaze-Ultra tasks,
likely because HILPs capture optimal, long-horizon behaviors.
Additionally, we note that, unlike OPAL or similar VAE-based approaches,
HILPs provide multiple zero-shot prompting schemes to query the learned policy to solve downstream tasks
without a separate high-level policy,
as shown in previous sections.
\begin{figure}[t!]
    \centering
    \vspace{-5pt}
    \includegraphics[width=0.9\linewidth]{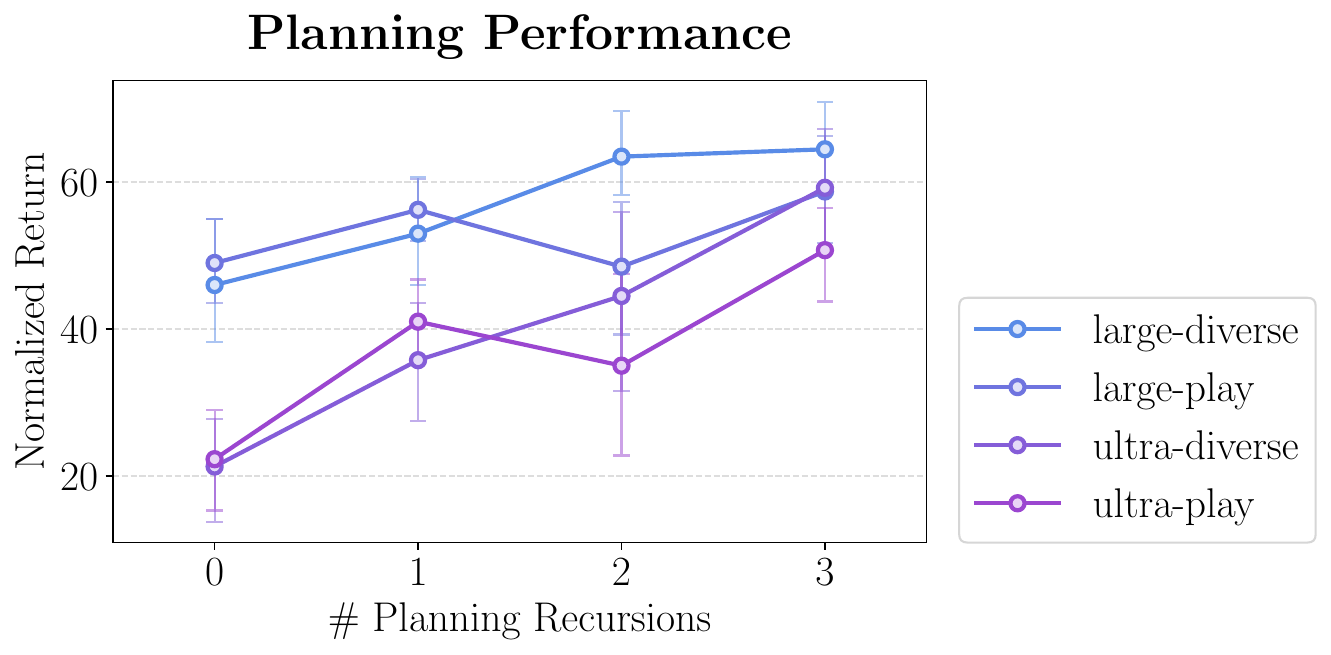}
    \vspace{-12pt}
    \caption{
    \footnotesize
    \textbf{Planning performance in AntMaze ($\mathbf{8}$ seeds).}
    Test-time midpoint planning (\Cref{sec:planning}) improves performance
    as the number of recursive refinements increases.
    }
    \label{fig:gcrl_plan}
    \vspace{-10pt}
\end{figure}

\cutsectionup
\section{Conclusion}
\cutsectiondown
In this work, we introduced Hilbert foundation policies (HILPs),
a general-purpose offline policy pre-training scheme
based on the idea of spanning a structured latent space that captures the temporal structure of the MDP.
We showed that structured Hilbert representations enable zero-shot prompting schemes for zero-shot RL and goal-conditioned RL
as well as test-time planning
to adapt pre-trained HILPs to downstream tasks.
Through our experiments,
we demonstrated that our single HILP framework often outperforms previous specialized methods for zero-shot RL,
goal-conditioned RL, and hierarchical RL.

\begin{table}[t!]
    \vspace{-10pt}
    \caption{
    \footnotesize
    \textbf{Hierarchical RL performance ($\mathbf{8}$ seeds).}
    HILP outperforms OPAL, a previous VAE-based offline skill learning method.
    }
    \vspace{5pt}
    \label{table:hrl}
    \centering
    \scalebox{0.92}
    {
    \begin{tabular}{lrrrrrr}
    \toprule
    Dataset & OPAL & \textbf{HILP} (\textbf{ours})\\
    \midrule
antmaze-large-diverse & $\mathbf{59.2}$ {\tiny $\pm 12.5$} & $56.0$ {\tiny $\pm 9.4$} \\
antmaze-large-play & $\mathbf{58.0}$ {\tiny $\pm 9.6$} & $\mathbf{58.0}$ {\tiny $\pm 11.1$} \\
antmaze-ultra-diverse & $21.0$ {\tiny $\pm 10.1$} & $\mathbf{43.8}$ {\tiny $\pm 8.6$} \\
antmaze-ultra-play & $22.8$ {\tiny $\pm 7.6$} & $\mathbf{47.8}$ {\tiny $\pm 13.2$} \\
kitchen-partial & $\mathbf{54.8}$ {\tiny $\pm 13.3$} & $\mathbf{54.1}$ {\tiny $\pm 3.9$} \\
kitchen-mixed & $\mathbf{58.2}$ {\tiny $\pm 7.5$} & $48.2$ {\tiny $\pm 10.1$} \\
    \midrule
\textbf{{Average}} & $45.7$ & $\mathbf{51.3}$ \\
    \bottomrule
    \end{tabular}
    }
    \vspace{-10pt}
\end{table}

\textbf{Final remarks.}
Offline unsupervised policy pre-training is all about determining and prioritizing the right behaviors to capture from offline data.
Prior works have proposed simply cloning dataset actions, capturing goal-reaching behaviors, or learning to maximize linear combinations of state features.
In this work, we proposed capturing long-horizon, state-spanning behaviors.
This is desirable because such global behaviors are usually harder to learn than local behaviors, and thus are worth capturing during pre-training.
However, one may still wonder:
``What if the environment is stochastic or partially observable, where an isometric embedding does not exist?''
``Are directional latent movements sufficient?''
``Is zero-shot task adaptation (without fine-tuning) the right way to \emph{use} learned behaviors?''
These are all important and valuable questions, and finding satisfying answers to these questions would lead to exciting future work.
For example, we may learn a \emph{local isometry}~\citep{riemannian_lee2018} instead of a global isometry
to handle partially observable environments,
learn more diverse latent movements to enhance expressivity,
or explore fine-tuning or few-shot learning for better task adaptation.
We further discuss limitations and future work in \Cref{sec:limit}.
Nonetheless, we hope that this work represents a step toward ideal offline unsupervised policy pre-training.

\cutsectionup
\section*{Impact Statement}
\cutsectiondown
This paper presents work whose goal is to advance the field of Machine Learning. There are many potential societal consequences of our work, none of which we feel must be specifically highlighted here.

\cutsectionup
\section*{Acknowledgments}
\cutsectiondown
We would like to thank Oleg Rybkin for an informative discussion,
and Qiyang Li and Yifei Zhou for helpful feedback on earlier drafts of this work.
This work was partly supported by the Korea Foundation for Advanced Studies (KFAS),
Berkeley Chancellor's Fellowship,
ONR N00014-21-1-2838, AFOSR FA9550-22-1-0273, and Qualcomm.
This research used the Savio computational cluster resource provided by the Berkeley Research Computing program at UC Berkeley.

\bibliography{icml2024}
\bibliographystyle{icml2024}

\newpage
\appendix
\onecolumn

\section{Limitations}
\label{sec:limit}
As mentioned in \Cref{sec:rep},
one limitation of our Hilbert representation objective is that
a symmetric Hilbert space might not be expressive enough to capture arbitrary MDPs.
Although we empirically demonstrate that HILPs exhibit strong performances
in various complex, long-horizon simulated robotic environments,
they might struggle in highly asymmetric or disconnected MDPs (\eg, environments in which gravity plays a significant role),
where there might not exist a reasonable approximate isometry to a Hilbert space.
We believe this limitation may be resolved by combining the notion of an inner product
with a universal quasimetric embedding~\citep{dist_pitis2020,quasi_wang2023},
which we leave for future work.
Another limitation is that our value-based representation learning objective (\Cref{eq:obj_phi})
might be optimistically biased in stochastic or partially observable environments~\citep{hiql_park2023}.
We believe combining our method with history-conditioned policies or recurrent world models may be one possible solution to deal with such MDPs.
We also note that our experiments assume the state space and environment dynamics to be the same at evaluation time.
We leave applying HILPs to multi-environment or transfer learning settings for future work.
Finally, we use Euclidean spaces as Hilbert spaces for our experiments in this work.
We believe applying HILPs to more general Hilbert spaces,
such as the $L^2$ function space or reproducing kernel Hilbert spaces,
may significantly enhance the expressivity of Hilbert representations.

\begin{figure}[t!]
    \centering
    \includegraphics[width=0.95\linewidth]{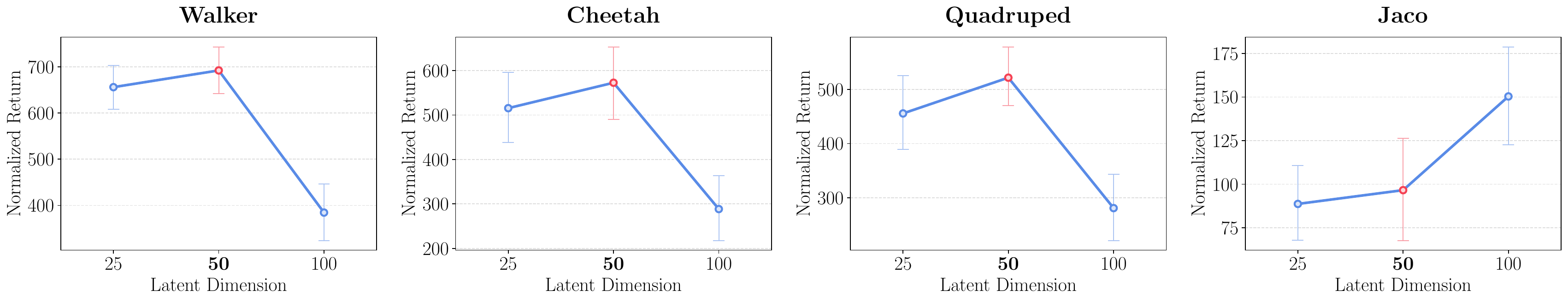}
    \vspace{-10pt}
    \caption{
    \footnotesize
    \textbf{Ablation study of latent dimensions for zero-shot RL.}
    In the ExORL benchmark, $D=50$ generally leads to the best performance across the environments.
    The results are aggregated over $4$ tasks, $4$ datasets, and $4$ seeds (\ie, $64$ values in total).
    }
    \label{fig:zs_abl}
\end{figure}

\begin{figure}[t!]
    \centering
    \includegraphics[width=0.95\linewidth]{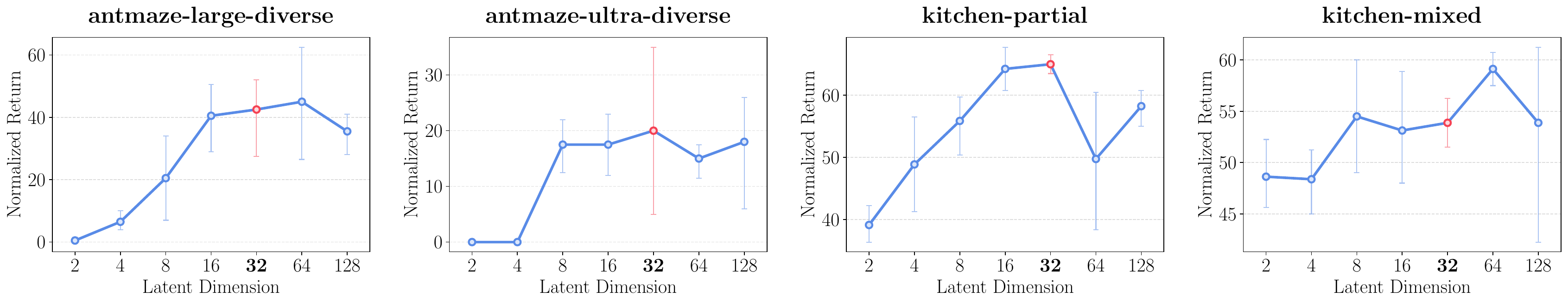}
    \vspace{-10pt}
    \caption{
    \footnotesize
    \textbf{Ablation study of latent dimensions for offline goal-conditioned RL ($\mathbf{4}$ seeds).}
    In D4RL tasks, $D=32$ generally leads to the best performance across the environments.
    }
    \label{fig:gcrl_dim_abl}
\end{figure}

\begin{figure}[t!]
    \centering
    \includegraphics[width=0.467\linewidth]{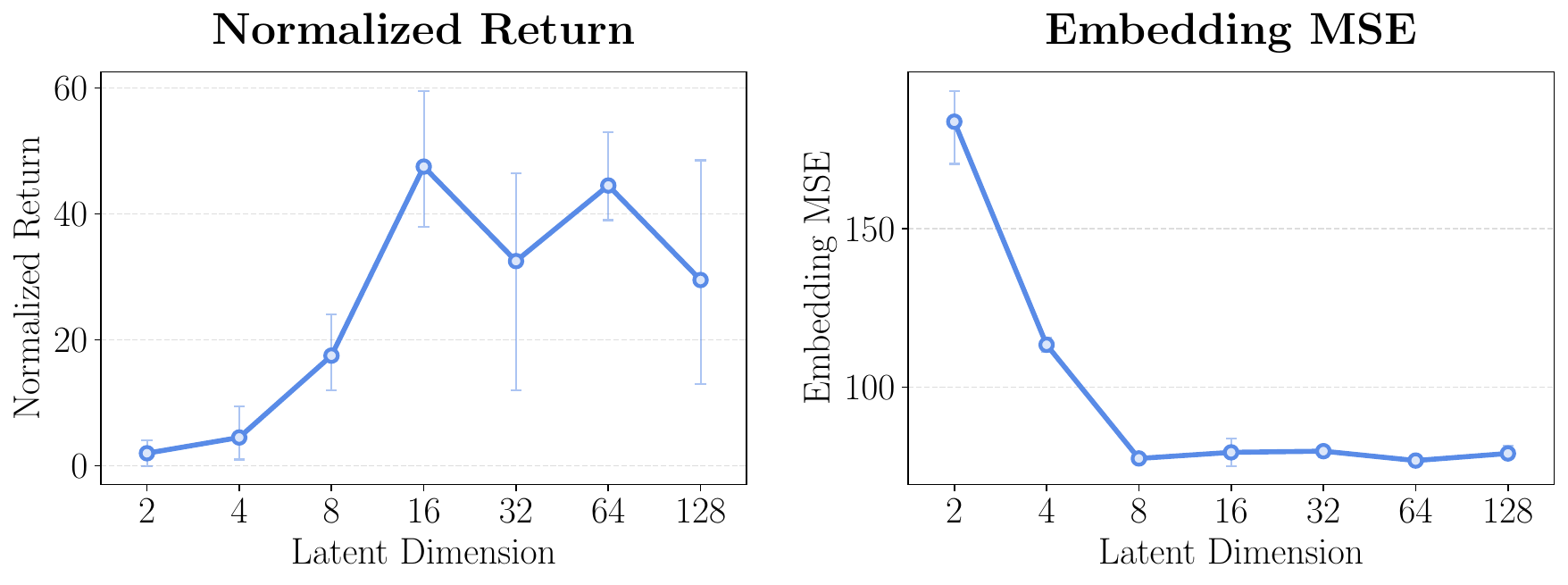}
    \vspace{-10pt}
    \caption{
    \footnotesize
    \textbf{Embedding error analysis ($\mathbf{4}$ seeds).}
    We show the relationship between performance and Hilbert embedding errors on the antmaze-large-diverse task.
    In general, lower embedding errors lead to better goal-reaching performance.
    }
    \label{fig:gcrl_dim_oracle_mse}
\end{figure}

\clearpage

\begin{figure}[t!]
    \centering
    \includegraphics[width=0.95\linewidth]{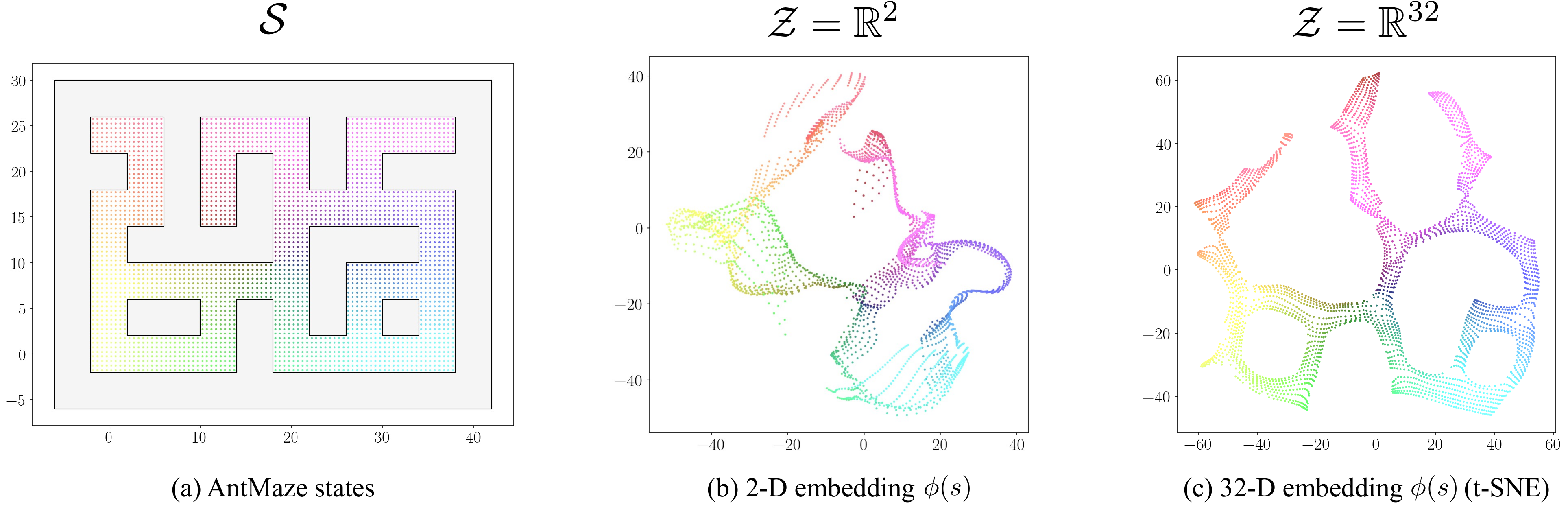}
    \vspace{-10pt}
    \caption{
    \footnotesize
    \textbf{Visualization of Hilbert representations.}
    We visualize Hilbert representations learned on the antmaze-large-diverse dataset.
    Since Hilbert representations are learned to capture the temporal structure of the MDP,
    they focus on the global layout of the maze even when we use a two-dimensional latent space ($D=2$),
    and accurately capture the maze layout with a $32$-dimensional latent space ($D=32$).
    }
    \label{fig:qual}
\end{figure}

\section{Additional Results}
\label{sec:add_results}

\textbf{Ablation study.}
\Cref{fig:gcrl_dim_abl,fig:zs_abl} show how the dimension $D$ of the latent space $\gZ = \sR^D$
affects the performances of zero-shot RL and offline goal-conditioned RL (without planning),
where we use $D=50$ for zero-shot RL and $D=32$ for goal-conditioned RL in our main experiments.
The results suggest that a latent dimension between $25$ and $64$
generally leads to the best performance in both cases.

\textbf{Embedding error analysis.}
To understand the relationship between Hilbert embedding errors and goal-reaching performance,
we compare the mean squared errors (MSEs) of Hilbert representations (\ie, $\E[(d^*(s, g) - \|\phi(s) - \phi(g)\|)^2]$)
and the final performances with different embedding dimensions on the antmaze-large-diverse task.
To approximate the ground-truth temporal distance $d^*(s, g)$ in practice,
we employ a monolithic goal-conditioned value function $V(s, g) \approx -d^*(s, g)$
trained with a separate goal-conditioned IQL objective.
We use the same IQL expectile of $0.9$ for both value functions in this analysis.
\Cref{fig:gcrl_dim_oracle_mse} shows the results,
suggesting that low embedding errors generally lead to better goal-reaching performances,
as predicted by \Cref{thm:gcrl}.

\textbf{Visualization of Hilbert representations.}
We train Hilbert representations with two different latent dimensions ($D \in \{2, 32\}$)
on the antmaze-large-diverse dataset,
and visualize the learned latent spaces in \Cref{fig:qual}.
We use a $t$-distributed stochastic neighbor embedding ($t$-SNE) to visualize $32$-dimensional latent states.
Since Hilbert representations are learned to preserve the temporal structure of the underlying environment,
they focus on the global layout of the maze even when we use a very low-dimensional latent space ($D=2$, \Cref{fig:qual}b),
and accurately capture the layout with $D=32$ (\Cref{fig:qual}c). 

\textbf{Full results on the ExORL benchmark.}
\Cref{table:exorl_full,table:exorl_pixel_full} present the full results (\emph{unnormalized} returns)
on the state- and pixel-based ExORL benchmarks.
\Cref{fig:zs_envs2,fig:zs_tasks,fig:zs_urls,fig:zs_pixel_envs2,fig:zs_pixel_tasks,fig:zs_pixel_urls}
show plots with three different aggregation criteria (per-environment, per-dataset, and per-task)
on the state- and pixel-based ExORL benchmarks.

\clearpage

\begin{table*}
    \caption{
    \footnotesize
    \textbf{Full results on the state-based ExORL benchmark ($\mathbf{4}$ seeds).}
    The table shows the \emph{unnormalized} return averaged over four seeds in each setting.
    }
    \vspace{5pt}
    \label{table:exorl_full}
    \centering
    \scalebox{0.69}{
    \setlength{\tabcolsep}{5pt}
    \begin{tabular}{lllrrrrrrrrrr}
    \toprule
    Dataset & Environment & Task & GC-TD3 & CL & Rand & IDM & AE & Lap & FDM & FB & \textbf{HILP} (\textbf{ours}) & \textbf{HILP-G} (\textbf{ours}) \\
    \midrule
\multirow{16}{*}{APS} & \multirow{4}{*}{Walker} & Flip & $406$ {\tiny $\pm 153$} & $137$ {\tiny $\pm 66$} & $92$ {\tiny $\pm 41$} & $354$ {\tiny $\pm 181$} & $289$ {\tiny $\pm 37$} & $390$ {\tiny $\pm 128$} & $426$ {\tiny $\pm 68$} & $334$ {\tiny $\pm 178$} & $573$ {\tiny $\pm 37$} & - \\
& & Run & $274$ {\tiny $\pm 59$} & $75$ {\tiny $\pm 33$} & $79$ {\tiny $\pm 18$} & $357$ {\tiny $\pm 78$} & $140$ {\tiny $\pm 18$} & $217$ {\tiny $\pm 72$} & $248$ {\tiny $\pm 23$} & $388$ {\tiny $\pm 27$} & $348$ {\tiny $\pm 14$} & - \\
& & Stand & $853$ {\tiny $\pm 78$} & $438$ {\tiny $\pm 64$} & $372$ {\tiny $\pm 76$} & $846$ {\tiny $\pm 113$} & $694$ {\tiny $\pm 54$} & $637$ {\tiny $\pm 150$} & $865$ {\tiny $\pm 77$} & $824$ {\tiny $\pm 54$} & $883$ {\tiny $\pm 42$} & - \\
& & Walk & $627$ {\tiny $\pm 201$} & $185$ {\tiny $\pm 112$} & $221$ {\tiny $\pm 68$} & $596$ {\tiny $\pm 70$} & $323$ {\tiny $\pm 131$} & $495$ {\tiny $\pm 84$} & $634$ {\tiny $\pm 91$} & $842$ {\tiny $\pm 105$} & $862$ {\tiny $\pm 31$} & - \\
\cmidrule{2-13}
& \multirow{4}{*}{Cheetah} & Run & $133$ {\tiny $\pm 95$} & $8$ {\tiny $\pm 8$} & $165$ {\tiny $\pm 82$} & $430$ {\tiny $\pm 24$} & $119$ {\tiny $\pm 52$} & $263$ {\tiny $\pm 28$} & $116$ {\tiny $\pm 116$} & $250$ {\tiny $\pm 135$} & $373$ {\tiny $\pm 72$} & - \\
& & Run Backward & $156$ {\tiny $\pm 137$} & $38$ {\tiny $\pm 24$} & $192$ {\tiny $\pm 31$} & $466$ {\tiny $\pm 30$} & $335$ {\tiny $\pm 30$} & $240$ {\tiny $\pm 18$} & $360$ {\tiny $\pm 17$} & $251$ {\tiny $\pm 39$} & $316$ {\tiny $\pm 21$} & - \\
& & Walk & $695$ {\tiny $\pm 334$} & $32$ {\tiny $\pm 45$} & $633$ {\tiny $\pm 295$} & $988$ {\tiny $\pm 1$} & $414$ {\tiny $\pm 161$} & $964$ {\tiny $\pm 26$} & $396$ {\tiny $\pm 287$} & $683$ {\tiny $\pm 267$} & $939$ {\tiny $\pm 55$} & - \\
& & Walk Backward & $930$ {\tiny $\pm 33$} & $197$ {\tiny $\pm 121$} & $905$ {\tiny $\pm 62$} & $986$ {\tiny $\pm 1$} & $973$ {\tiny $\pm 14$} & $984$ {\tiny $\pm 1$} & $982$ {\tiny $\pm 2$} & $980$ {\tiny $\pm 3$} & $985$ {\tiny $\pm 2$} & - \\
\cmidrule{2-13}
& \multirow{4}{*}{Quadruped} & Jump & $732$ {\tiny $\pm 62$} & $78$ {\tiny $\pm 67$} & $780$ {\tiny $\pm 14$} & $144$ {\tiny $\pm 97$} & $740$ {\tiny $\pm 58$} & $696$ {\tiny $\pm 68$} & $707$ {\tiny $\pm 30$} & $757$ {\tiny $\pm 52$} & $623$ {\tiny $\pm 149$} & - \\
& & Run & $420$ {\tiny $\pm 37$} & $77$ {\tiny $\pm 90$} & $486$ {\tiny $\pm 3$} & $87$ {\tiny $\pm 36$} & $481$ {\tiny $\pm 11$} & $483$ {\tiny $\pm 12$} & $481$ {\tiny $\pm 4$} & $474$ {\tiny $\pm 33$} & $411$ {\tiny $\pm 62$} & - \\
& & Stand & $938$ {\tiny $\pm 32$} & $131$ {\tiny $\pm 128$} & $965$ {\tiny $\pm 4$} & $177$ {\tiny $\pm 106$} & $944$ {\tiny $\pm 19$} & $914$ {\tiny $\pm 65$} & $961$ {\tiny $\pm 20$} & $949$ {\tiny $\pm 30$} & $797$ {\tiny $\pm 117$} & - \\
& & Walk & $486$ {\tiny $\pm 53$} & $91$ {\tiny $\pm 106$} & $513$ {\tiny $\pm 51$} & $99$ {\tiny $\pm 40$} & $600$ {\tiny $\pm 131$} & $550$ {\tiny $\pm 53$} & $578$ {\tiny $\pm 145$} & $584$ {\tiny $\pm 12$} & $605$ {\tiny $\pm 75$} & - \\
\cmidrule{2-13}
& \multirow{4}{*}{Jaco} & Reach Bottom Left & $89$ {\tiny $\pm 53$} & $1$ {\tiny $\pm 1$} & $5$ {\tiny $\pm 6$} & $37$ {\tiny $\pm 20$} & $1$ {\tiny $\pm 0$} & $16$ {\tiny $\pm 6$} & $12$ {\tiny $\pm 18$} & $14$ {\tiny $\pm 12$} & $88$ {\tiny $\pm 41$} & $78$ {\tiny $\pm 11$} \\
& & Reach Bottom Right & $121$ {\tiny $\pm 80$} & $0$ {\tiny $\pm 0$} & $5$ {\tiny $\pm 9$} & $31$ {\tiny $\pm 14$} & $6$ {\tiny $\pm 5$} & $10$ {\tiny $\pm 4$} & $28$ {\tiny $\pm 20$} & $24$ {\tiny $\pm 7$} & $48$ {\tiny $\pm 24$} & $84$ {\tiny $\pm 14$} \\
& & Reach Top Left & $71$ {\tiny $\pm 48$} & $1$ {\tiny $\pm 1$} & $3$ {\tiny $\pm 3$} & $20$ {\tiny $\pm 16$} & $5$ {\tiny $\pm 6$} & $51$ {\tiny $\pm 40$} & $21$ {\tiny $\pm 16$} & $23$ {\tiny $\pm 17$} & $49$ {\tiny $\pm 18$} & $75$ {\tiny $\pm 3$} \\
& & Reach Top Right & $71$ {\tiny $\pm 74$} & $1$ {\tiny $\pm 1$} & $7$ {\tiny $\pm 7$} & $24$ {\tiny $\pm 21$} & $12$ {\tiny $\pm 20$} & $26$ {\tiny $\pm 15$} & $34$ {\tiny $\pm 40$} & $17$ {\tiny $\pm 15$} & $51$ {\tiny $\pm 32$} & $80$ {\tiny $\pm 13$} \\
\cmidrule{1-13}
\multirow{16}{*}{APT} & \multirow{4}{*}{Walker} & Flip & $413$ {\tiny $\pm 44$} & $162$ {\tiny $\pm 61$} & $99$ {\tiny $\pm 14$} & $480$ {\tiny $\pm 48$} & $486$ {\tiny $\pm 53$} & $575$ {\tiny $\pm 37$} & $519$ {\tiny $\pm 80$} & $526$ {\tiny $\pm 89$} & $714$ {\tiny $\pm 89$} & - \\
& & Run & $187$ {\tiny $\pm 19$} & $92$ {\tiny $\pm 27$} & $88$ {\tiny $\pm 4$} & $328$ {\tiny $\pm 26$} & $284$ {\tiny $\pm 22$} & $301$ {\tiny $\pm 30$} & $338$ {\tiny $\pm 69$} & $386$ {\tiny $\pm 13$} & $440$ {\tiny $\pm 39$} & - \\
& & Stand & $757$ {\tiny $\pm 277$} & $531$ {\tiny $\pm 121$} & $593$ {\tiny $\pm 139$} & $862$ {\tiny $\pm 29$} & $866$ {\tiny $\pm 85$} & $791$ {\tiny $\pm 104$} & $873$ {\tiny $\pm 44$} & $884$ {\tiny $\pm 10$} & $877$ {\tiny $\pm 68$} & - \\
& & Walk & $673$ {\tiny $\pm 256$} & $138$ {\tiny $\pm 52$} & $232$ {\tiny $\pm 156$} & $640$ {\tiny $\pm 160$} & $775$ {\tiny $\pm 91$} & $685$ {\tiny $\pm 94$} & $719$ {\tiny $\pm 145$} & $891$ {\tiny $\pm 41$} & $843$ {\tiny $\pm 57$} & - \\
\cmidrule{2-13}
& \multirow{4}{*}{Cheetah} & Run & $101$ {\tiny $\pm 126$} & $68$ {\tiny $\pm 19$} & $51$ {\tiny $\pm 17$} & $398$ {\tiny $\pm 107$} & $79$ {\tiny $\pm 26$} & $172$ {\tiny $\pm 60$} & $194$ {\tiny $\pm 75$} & $141$ {\tiny $\pm 91$} & $269$ {\tiny $\pm 69$} & - \\
& & Run Backward & $38$ {\tiny $\pm 19$} & $26$ {\tiny $\pm 10$} & $30$ {\tiny $\pm 7$} & $331$ {\tiny $\pm 86$} & $146$ {\tiny $\pm 41$} & $157$ {\tiny $\pm 81$} & $188$ {\tiny $\pm 105$} & $49$ {\tiny $\pm 9$} & $157$ {\tiny $\pm 98$} & - \\
& & Walk & $475$ {\tiny $\pm 350$} & $325$ {\tiny $\pm 76$} & $218$ {\tiny $\pm 42$} & $875$ {\tiny $\pm 131$} & $326$ {\tiny $\pm 139$} & $703$ {\tiny $\pm 235$} & $684$ {\tiny $\pm 215$} & $439$ {\tiny $\pm 307$} & $808$ {\tiny $\pm 142$} & - \\
& & Walk Backward & $265$ {\tiny $\pm 66$} & $128$ {\tiny $\pm 49$} & $149$ {\tiny $\pm 50$} & $946$ {\tiny $\pm 59$} & $455$ {\tiny $\pm 114$} & $642$ {\tiny $\pm 238$} & $646$ {\tiny $\pm 199$} & $208$ {\tiny $\pm 43$} & $779$ {\tiny $\pm 229$} & - \\
\cmidrule{2-13}
& \multirow{4}{*}{Quadruped} & Jump & $608$ {\tiny $\pm 101$} & $281$ {\tiny $\pm 166$} & $686$ {\tiny $\pm 49$} & $167$ {\tiny $\pm 67$} & $593$ {\tiny $\pm 56$} & $743$ {\tiny $\pm 61$} & $722$ {\tiny $\pm 40$} & $738$ {\tiny $\pm 45$} & $686$ {\tiny $\pm 9$} & - \\
& & Run & $389$ {\tiny $\pm 70$} & $196$ {\tiny $\pm 130$} & $458$ {\tiny $\pm 22$} & $120$ {\tiny $\pm 37$} & $415$ {\tiny $\pm 41$} & $452$ {\tiny $\pm 21$} & $454$ {\tiny $\pm 37$} & $424$ {\tiny $\pm 73$} & $461$ {\tiny $\pm 16$} & - \\
& & Stand & $918$ {\tiny $\pm 77$} & $379$ {\tiny $\pm 246$} & $914$ {\tiny $\pm 44$} & $155$ {\tiny $\pm 48$} & $835$ {\tiny $\pm 82$} & $879$ {\tiny $\pm 75$} & $919$ {\tiny $\pm 34$} & $891$ {\tiny $\pm 50$} & $930$ {\tiny $\pm 30$} & - \\
& & Walk & $441$ {\tiny $\pm 43$} & $191$ {\tiny $\pm 104$} & $465$ {\tiny $\pm 18$} & $94$ {\tiny $\pm 23$} & $417$ {\tiny $\pm 42$} & $433$ {\tiny $\pm 75$} & $513$ {\tiny $\pm 64$} & $427$ {\tiny $\pm 35$} & $468$ {\tiny $\pm 22$} & - \\
\cmidrule{2-13}
& \multirow{4}{*}{Jaco} & Reach Bottom Left & $1$ {\tiny $\pm 1$} & $0$ {\tiny $\pm 0$} & $2$ {\tiny $\pm 2$} & $34$ {\tiny $\pm 12$} & $5$ {\tiny $\pm 9$} & $3$ {\tiny $\pm 4$} & $75$ {\tiny $\pm 30$} & $7$ {\tiny $\pm 8$} & $12$ {\tiny $\pm 17$} & $30$ {\tiny $\pm 16$} \\
& & Reach Bottom Right & $0$ {\tiny $\pm 1$} & $0$ {\tiny $\pm 0$} & $1$ {\tiny $\pm 1$} & $47$ {\tiny $\pm 14$} & $3$ {\tiny $\pm 3$} & $9$ {\tiny $\pm 13$} & $56$ {\tiny $\pm 22$} & $8$ {\tiny $\pm 3$} & $29$ {\tiny $\pm 44$} & $37$ {\tiny $\pm 14$} \\
& & Reach Top Left & $0$ {\tiny $\pm 0$} & $0$ {\tiny $\pm 0$} & $9$ {\tiny $\pm 15$} & $19$ {\tiny $\pm 5$} & $7$ {\tiny $\pm 13$} & $1$ {\tiny $\pm 1$} & $13$ {\tiny $\pm 12$} & $28$ {\tiny $\pm 17$} & $5$ {\tiny $\pm 4$} & $30$ {\tiny $\pm 14$} \\
& & Reach Top Right & $0$ {\tiny $\pm 0$} & $0$ {\tiny $\pm 0$} & $1$ {\tiny $\pm 1$} & $40$ {\tiny $\pm 10$} & $3$ {\tiny $\pm 6$} & $2$ {\tiny $\pm 3$} & $12$ {\tiny $\pm 4$} & $13$ {\tiny $\pm 10$} & $21$ {\tiny $\pm 24$} & $34$ {\tiny $\pm 11$} \\
\cmidrule{1-13}
\multirow{16}{*}{Proto} & \multirow{4}{*}{Walker} & Flip & $494$ {\tiny $\pm 153$} & $216$ {\tiny $\pm 61$} & $159$ {\tiny $\pm 41$} & $432$ {\tiny $\pm 62$} & $553$ {\tiny $\pm 105$} & $531$ {\tiny $\pm 104$} & $433$ {\tiny $\pm 64$} & $560$ {\tiny $\pm 94$} & $675$ {\tiny $\pm 62$} & - \\
& & Run & $324$ {\tiny $\pm 67$} & $141$ {\tiny $\pm 26$} & $68$ {\tiny $\pm 36$} & $251$ {\tiny $\pm 42$} & $312$ {\tiny $\pm 47$} & $352$ {\tiny $\pm 17$} & $296$ {\tiny $\pm 43$} & $425$ {\tiny $\pm 49$} & $402$ {\tiny $\pm 28$} & - \\
& & Stand & $796$ {\tiny $\pm 179$} & $629$ {\tiny $\pm 108$} & $407$ {\tiny $\pm 110$} & $903$ {\tiny $\pm 73$} & $916$ {\tiny $\pm 26$} & $897$ {\tiny $\pm 65$} & $900$ {\tiny $\pm 42$} & $844$ {\tiny $\pm 103$} & $930$ {\tiny $\pm 27$} & - \\
& & Walk & $866$ {\tiny $\pm 33$} & $383$ {\tiny $\pm 57$} & $190$ {\tiny $\pm 214$} & $588$ {\tiny $\pm 195$} & $851$ {\tiny $\pm 37$} & $864$ {\tiny $\pm 50$} & $873$ {\tiny $\pm 63$} & $905$ {\tiny $\pm 29$} & $905$ {\tiny $\pm 32$} & - \\
\cmidrule{2-13}
& \multirow{4}{*}{Cheetah} & Run & $135$ {\tiny $\pm 89$} & $4$ {\tiny $\pm 3$} & $123$ {\tiny $\pm 92$} & $370$ {\tiny $\pm 59$} & $160$ {\tiny $\pm 121$} & $282$ {\tiny $\pm 37$} & $363$ {\tiny $\pm 24$} & $223$ {\tiny $\pm 42$} & $227$ {\tiny $\pm 27$} & - \\
& & Run Backward & $173$ {\tiny $\pm 34$} & $5$ {\tiny $\pm 3$} & $177$ {\tiny $\pm 33$} & $383$ {\tiny $\pm 23$} & $254$ {\tiny $\pm 38$} & $218$ {\tiny $\pm 10$} & $309$ {\tiny $\pm 28$} & $151$ {\tiny $\pm 38$} & $234$ {\tiny $\pm 13$} & - \\
& & Walk & $923$ {\tiny $\pm 81$} & $23$ {\tiny $\pm 18$} & $556$ {\tiny $\pm 362$} & $939$ {\tiny $\pm 63$} & $563$ {\tiny $\pm 363$} & $982$ {\tiny $\pm 9$} & $944$ {\tiny $\pm 60$} & $949$ {\tiny $\pm 40$} & $973$ {\tiny $\pm 12$} & - \\
& & Walk Backward & $560$ {\tiny $\pm 368$} & $28$ {\tiny $\pm 17$} & $820$ {\tiny $\pm 136$} & $987$ {\tiny $\pm 1$} & $916$ {\tiny $\pm 106$} & $981$ {\tiny $\pm 6$} & $984$ {\tiny $\pm 1$} & $737$ {\tiny $\pm 160$} & $985$ {\tiny $\pm 2$} & - \\
\cmidrule{2-13}
& \multirow{4}{*}{Quadruped} & Jump & $298$ {\tiny $\pm 60$} & $27$ {\tiny $\pm 20$} & $176$ {\tiny $\pm 59$} & $69$ {\tiny $\pm 39$} & $289$ {\tiny $\pm 35$} & $214$ {\tiny $\pm 27$} & $287$ {\tiny $\pm 61$} & $231$ {\tiny $\pm 75$} & $282$ {\tiny $\pm 105$} & - \\
& & Run & $176$ {\tiny $\pm 103$} & $17$ {\tiny $\pm 15$} & $127$ {\tiny $\pm 29$} & $51$ {\tiny $\pm 24$} & $187$ {\tiny $\pm 27$} & $181$ {\tiny $\pm 52$} & $243$ {\tiny $\pm 37$} & $126$ {\tiny $\pm 26$} & $227$ {\tiny $\pm 42$} & - \\
& & Stand & $436$ {\tiny $\pm 45$} & $35$ {\tiny $\pm 28$} & $307$ {\tiny $\pm 35$} & $117$ {\tiny $\pm 60$} & $403$ {\tiny $\pm 48$} & $287$ {\tiny $\pm 52$} & $453$ {\tiny $\pm 63$} & $262$ {\tiny $\pm 16$} & $425$ {\tiny $\pm 156$} & - \\
& & Walk & $237$ {\tiny $\pm 39$} & $16$ {\tiny $\pm 15$} & $120$ {\tiny $\pm 38$} & $58$ {\tiny $\pm 11$} & $229$ {\tiny $\pm 25$} & $175$ {\tiny $\pm 60$} & $238$ {\tiny $\pm 56$} & $224$ {\tiny $\pm 123$} & $136$ {\tiny $\pm 50$} & - \\
\cmidrule{2-13}
& \multirow{4}{*}{Jaco} & Reach Bottom Left & $10$ {\tiny $\pm 12$} & $1$ {\tiny $\pm 2$} & $1$ {\tiny $\pm 1$} & $53$ {\tiny $\pm 19$} & $11$ {\tiny $\pm 8$} & $3$ {\tiny $\pm 4$} & $60$ {\tiny $\pm 28$} & $15$ {\tiny $\pm 24$} & $3$ {\tiny $\pm 4$} & $60$ {\tiny $\pm 14$} \\
& & Reach Bottom Right & $1$ {\tiny $\pm 1$} & $0$ {\tiny $\pm 0$} & $2$ {\tiny $\pm 3$} & $57$ {\tiny $\pm 28$} & $29$ {\tiny $\pm 35$} & $2$ {\tiny $\pm 2$} & $80$ {\tiny $\pm 31$} & $16$ {\tiny $\pm 31$} & $7$ {\tiny $\pm 12$} & $61$ {\tiny $\pm 4$} \\
& & Reach Top Left & $5$ {\tiny $\pm 3$} & $2$ {\tiny $\pm 2$} & $4$ {\tiny $\pm 6$} & $37$ {\tiny $\pm 11$} & $9$ {\tiny $\pm 8$} & $1$ {\tiny $\pm 0$} & $17$ {\tiny $\pm 17$} & $11$ {\tiny $\pm 5$} & $13$ {\tiny $\pm 12$} & $60$ {\tiny $\pm 5$} \\
& & Reach Top Right & $17$ {\tiny $\pm 13$} & $1$ {\tiny $\pm 1$} & $0$ {\tiny $\pm 0$} & $34$ {\tiny $\pm 16$} & $22$ {\tiny $\pm 21$} & $2$ {\tiny $\pm 3$} & $32$ {\tiny $\pm 27$} & $39$ {\tiny $\pm 40$} & $12$ {\tiny $\pm 9$} & $61$ {\tiny $\pm 4$} \\
\cmidrule{1-13}
\multirow{16}{*}{RND} & \multirow{4}{*}{Walker} & Flip & $298$ {\tiny $\pm 110$} & $115$ {\tiny $\pm 16$} & $252$ {\tiny $\pm 43$} & $453$ {\tiny $\pm 21$} & $499$ {\tiny $\pm 33$} & $563$ {\tiny $\pm 40$} & $416$ {\tiny $\pm 20$} & $548$ {\tiny $\pm 94$} & $563$ {\tiny $\pm 136$} & - \\
& & Run & $166$ {\tiny $\pm 27$} & $59$ {\tiny $\pm 19$} & $75$ {\tiny $\pm 47$} & $205$ {\tiny $\pm 69$} & $259$ {\tiny $\pm 18$} & $311$ {\tiny $\pm 29$} & $295$ {\tiny $\pm 70$} & $409$ {\tiny $\pm 15$} & $401$ {\tiny $\pm 30$} & - \\
& & Stand & $812$ {\tiny $\pm 133$} & $309$ {\tiny $\pm 100$} & $397$ {\tiny $\pm 50$} & $778$ {\tiny $\pm 52$} & $806$ {\tiny $\pm 44$} & $844$ {\tiny $\pm 35$} & $821$ {\tiny $\pm 84$} & $866$ {\tiny $\pm 120$} & $800$ {\tiny $\pm 61$} & - \\
& & Walk & $703$ {\tiny $\pm 275$} & $99$ {\tiny $\pm 35$} & $149$ {\tiny $\pm 145$} & $401$ {\tiny $\pm 160$} & $676$ {\tiny $\pm 135$} & $811$ {\tiny $\pm 77$} & $476$ {\tiny $\pm 259$} & $811$ {\tiny $\pm 52$} & $855$ {\tiny $\pm 34$} & - \\
\cmidrule{2-13}
& \multirow{4}{*}{Cheetah} & Run & $65$ {\tiny $\pm 66$} & $49$ {\tiny $\pm 34$} & $49$ {\tiny $\pm 24$} & $139$ {\tiny $\pm 67$} & $116$ {\tiny $\pm 12$} & $124$ {\tiny $\pm 32$} & $107$ {\tiny $\pm 26$} & $183$ {\tiny $\pm 83$} & $262$ {\tiny $\pm 53$} & - \\
& & Run Backward & $50$ {\tiny $\pm 35$} & $20$ {\tiny $\pm 6$} & $28$ {\tiny $\pm 10$} & $298$ {\tiny $\pm 27$} & $143$ {\tiny $\pm 86$} & $113$ {\tiny $\pm 37$} & $177$ {\tiny $\pm 62$} & $153$ {\tiny $\pm 41$} & $187$ {\tiny $\pm 55$} & - \\
& & Walk & $281$ {\tiny $\pm 116$} & $210$ {\tiny $\pm 123$} & $231$ {\tiny $\pm 121$} & $607$ {\tiny $\pm 181$} & $500$ {\tiny $\pm 114$} & $571$ {\tiny $\pm 111$} & $494$ {\tiny $\pm 122$} & $636$ {\tiny $\pm 291$} & $823$ {\tiny $\pm 141$} & - \\
& & Walk Backward & $146$ {\tiny $\pm 51$} & $130$ {\tiny $\pm 62$} & $151$ {\tiny $\pm 66$} & $956$ {\tiny $\pm 16$} & $528$ {\tiny $\pm 296$} & $536$ {\tiny $\pm 92$} & $652$ {\tiny $\pm 274$} & $677$ {\tiny $\pm 85$} & $843$ {\tiny $\pm 184$} & - \\
\cmidrule{2-13}
& \multirow{4}{*}{Quadruped} & Jump & $639$ {\tiny $\pm 106$} & $402$ {\tiny $\pm 363$} & $737$ {\tiny $\pm 48$} & $174$ {\tiny $\pm 48$} & $698$ {\tiny $\pm 94$} & $508$ {\tiny $\pm 182$} & $758$ {\tiny $\pm 98$} & $642$ {\tiny $\pm 36$} & $556$ {\tiny $\pm 101$} & - \\
& & Run & $435$ {\tiny $\pm 38$} & $289$ {\tiny $\pm 229$} & $458$ {\tiny $\pm 27$} & $103$ {\tiny $\pm 38$} & $446$ {\tiny $\pm 57$} & $346$ {\tiny $\pm 81$} & $491$ {\tiny $\pm 7$} & $436$ {\tiny $\pm 26$} & $393$ {\tiny $\pm 42$} & - \\
& & Stand & $910$ {\tiny $\pm 44$} & $555$ {\tiny $\pm 477$} & $934$ {\tiny $\pm 26$} & $240$ {\tiny $\pm 68$} & $831$ {\tiny $\pm 113$} & $681$ {\tiny $\pm 191$} & $971$ {\tiny $\pm 11$} & $797$ {\tiny $\pm 72$} & $810$ {\tiny $\pm 97$} & - \\
& & Walk & $470$ {\tiny $\pm 21$} & $303$ {\tiny $\pm 221$} & $541$ {\tiny $\pm 98$} & $92$ {\tiny $\pm 41$} & $492$ {\tiny $\pm 198$} & $441$ {\tiny $\pm 88$} & $601$ {\tiny $\pm 82$} & $642$ {\tiny $\pm 202$} & $542$ {\tiny $\pm 32$} & - \\
\cmidrule{2-13}
& \multirow{4}{*}{Jaco} & Reach Bottom Left & $1$ {\tiny $\pm 1$} & $0$ {\tiny $\pm 0$} & $1$ {\tiny $\pm 1$} & $60$ {\tiny $\pm 18$} & $1$ {\tiny $\pm 2$} & $18$ {\tiny $\pm 25$} & $53$ {\tiny $\pm 20$} & $18$ {\tiny $\pm 11$} & $19$ {\tiny $\pm 20$} & $46$ {\tiny $\pm 14$} \\
& & Reach Bottom Right & $1$ {\tiny $\pm 1$} & $0$ {\tiny $\pm 0$} & $2$ {\tiny $\pm 2$} & $45$ {\tiny $\pm 24$} & $1$ {\tiny $\pm 1$} & $7$ {\tiny $\pm 9$} & $38$ {\tiny $\pm 14$} & $29$ {\tiny $\pm 12$} & $18$ {\tiny $\pm 17$} & $55$ {\tiny $\pm 23$} \\
& & Reach Top Left & $0$ {\tiny $\pm 0$} & $0$ {\tiny $\pm 0$} & $0$ {\tiny $\pm 0$} & $34$ {\tiny $\pm 4$} & $6$ {\tiny $\pm 8$} & $5$ {\tiny $\pm 8$} & $36$ {\tiny $\pm 19$} & $45$ {\tiny $\pm 16$} & $8$ {\tiny $\pm 10$} & $47$ {\tiny $\pm 15$} \\
& & Reach Top Right & $0$ {\tiny $\pm 0$} & $1$ {\tiny $\pm 2$} & $1$ {\tiny $\pm 0$} & $22$ {\tiny $\pm 6$} & $6$ {\tiny $\pm 3$} & $8$ {\tiny $\pm 10$} & $44$ {\tiny $\pm 24$} & $22$ {\tiny $\pm 7$} & $5$ {\tiny $\pm 4$} & $38$ {\tiny $\pm 18$} \\
    \bottomrule
    \end{tabular}
    }
\end{table*}

\clearpage

\begin{figure}
    \centering
    \includegraphics[width=0.95\linewidth]{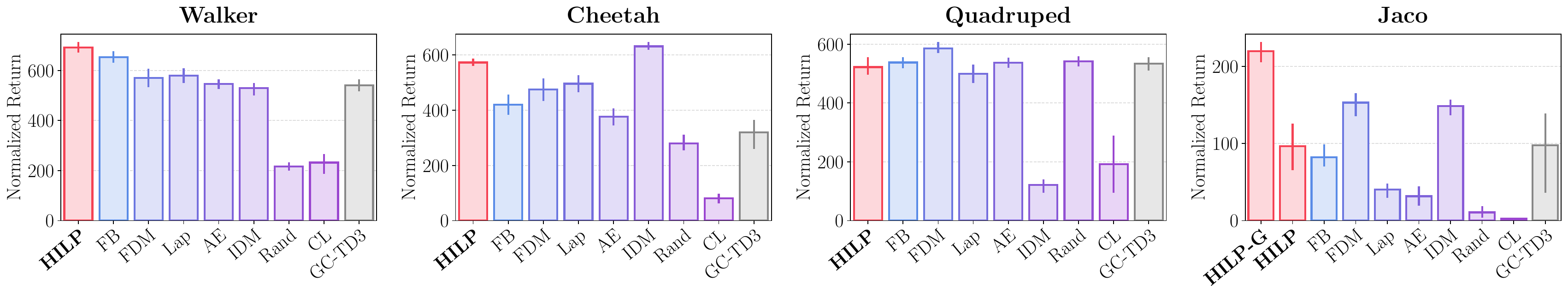}
    \vspace{-10pt}
    \caption{
    \footnotesize
    \textbf{Per-environment performances on the state-based ExORL benchmark.}
    The results are aggregated over $4$ tasks, $4$ datasets, and $4$ seeds (\ie, $64$ values in total).
    }
    \label{fig:zs_envs2}
\end{figure}
\begin{figure}
    \centering
    \includegraphics[width=0.95\linewidth]{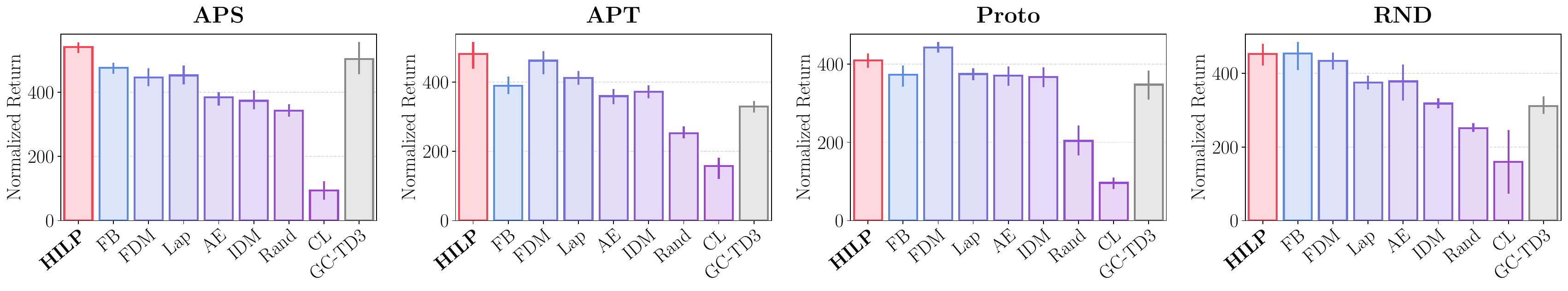}
    \vspace{-10pt}
    \caption{
    \footnotesize
    \textbf{Per-dataset performances on the state-based ExORL benchmark.}
    The results are aggregated over $4$ environments, $4$ tasks, and $4$ seeds (\ie, $64$ values in total).
    }
    \label{fig:zs_urls}
\end{figure}
\begin{figure}
    \centering
    \includegraphics[width=0.95\linewidth]{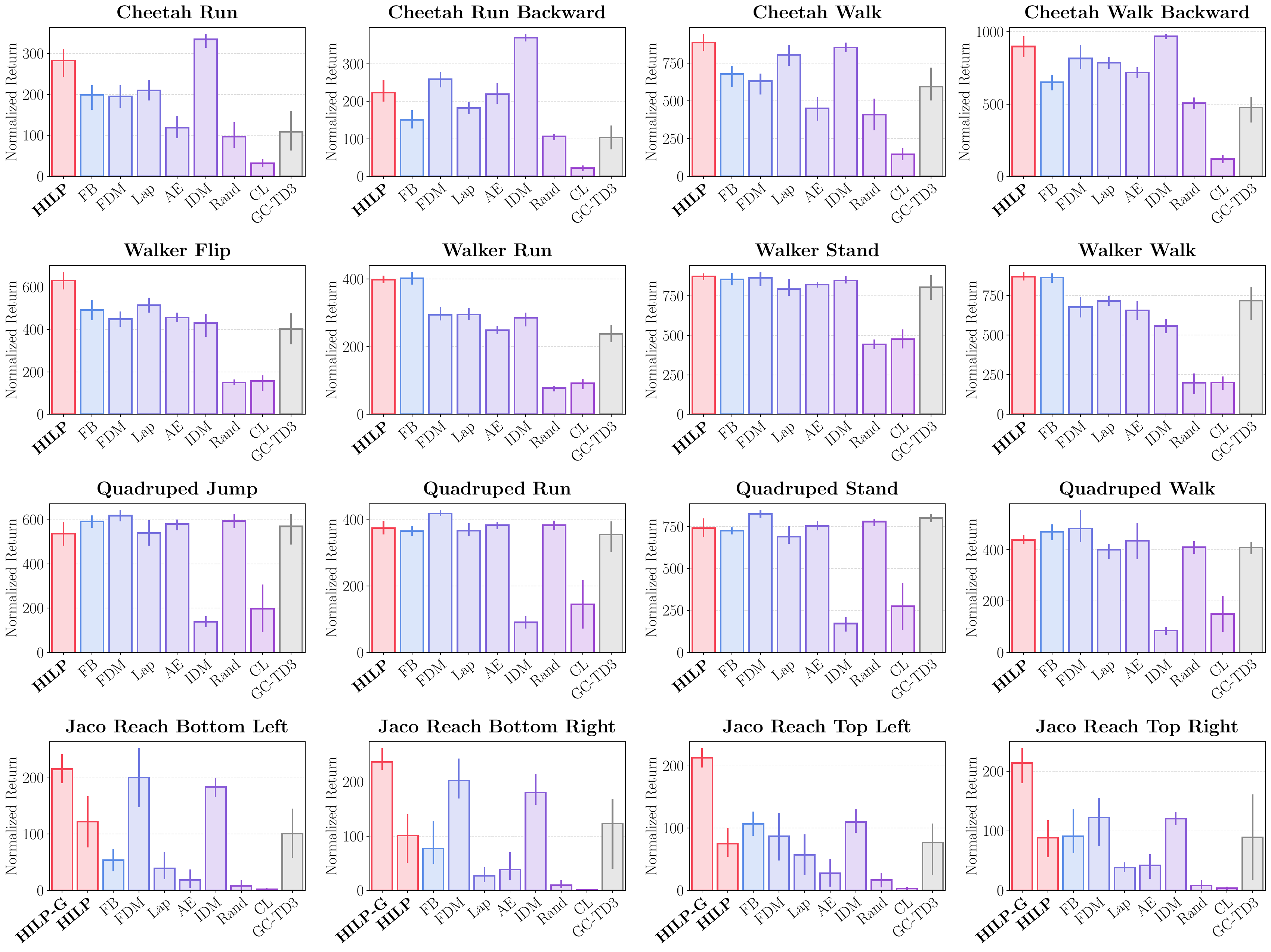}
    \vspace{-10pt}
    \caption{
    \footnotesize
    \textbf{Per-task performances on the state-based ExORL benchmark.}
    The results are aggregated over $4$ datasets and $4$ seeds (\ie, $16$ values in total).
    }
    \label{fig:zs_tasks}
\end{figure}

\clearpage

\begin{table*}
    \caption{
    \footnotesize
    \textbf{Full results on the pixel-based ExORL benchmark ($\mathbf{4}$ seeds).}
    The table shows the \emph{unnormalized} return averaged over four seeds in each setting.
    }
    \vspace{5pt}
    \label{table:exorl_pixel_full}
    \centering
    \scalebox{0.69}{
    \setlength{\tabcolsep}{5pt}
    \begin{tabular}{lllrrr}
    \toprule
    Dataset & Environment & Task & FDM & FB & \textbf{HILP} (\textbf{ours}) \\
    \midrule
\multirow{16}{*}{APS} & \multirow{4}{*}{Walker} & Flip & $158$ {\tiny $\pm 25$} & $66$ {\tiny $\pm 34$} & $127$ {\tiny $\pm 23$} \\
& & Run & $126$ {\tiny $\pm 26$} & $53$ {\tiny $\pm 18$} & $97$ {\tiny $\pm 2$} \\
& & Stand & $608$ {\tiny $\pm 90$} & $340$ {\tiny $\pm 122$} & $520$ {\tiny $\pm 30$} \\
& & Walk & $317$ {\tiny $\pm 156$} & $231$ {\tiny $\pm 109$} & $372$ {\tiny $\pm 68$} \\
\cmidrule{2-6}
& \multirow{4}{*}{Cheetah} & Run & $189$ {\tiny $\pm 41$} & $21$ {\tiny $\pm 34$} & $118$ {\tiny $\pm 113$} \\
& & Run Backward & $59$ {\tiny $\pm 73$} & $24$ {\tiny $\pm 28$} & $248$ {\tiny $\pm 46$} \\
& & Walk & $613$ {\tiny $\pm 77$} & $52$ {\tiny $\pm 91$} & $273$ {\tiny $\pm 383$} \\
& & Walk Backward & $371$ {\tiny $\pm 307$} & $89$ {\tiny $\pm 101$} & $967$ {\tiny $\pm 8$} \\
\cmidrule{2-6}
& \multirow{4}{*}{Quadruped} & Jump & $224$ {\tiny $\pm 23$} & $291$ {\tiny $\pm 34$} & $301$ {\tiny $\pm 31$} \\
& & Run & $172$ {\tiny $\pm 27$} & $231$ {\tiny $\pm 23$} & $204$ {\tiny $\pm 39$} \\
& & Stand & $343$ {\tiny $\pm 16$} & $444$ {\tiny $\pm 53$} & $397$ {\tiny $\pm 47$} \\
& & Walk & $176$ {\tiny $\pm 25$} & $230$ {\tiny $\pm 22$} & $195$ {\tiny $\pm 16$} \\
\cmidrule{2-6}
& \multirow{4}{*}{Jaco} & Reach Bottom Left & $12$ {\tiny $\pm 2$} & $52$ {\tiny $\pm 45$} & $63$ {\tiny $\pm 27$} \\
& & Reach Bottom Right & $29$ {\tiny $\pm 18$} & $53$ {\tiny $\pm 18$} & $68$ {\tiny $\pm 11$} \\
& & Reach Top Left & $21$ {\tiny $\pm 8$} & $31$ {\tiny $\pm 21$} & $62$ {\tiny $\pm 35$} \\
& & Reach Top Right & $36$ {\tiny $\pm 13$} & $53$ {\tiny $\pm 30$} & $62$ {\tiny $\pm 46$} \\
\cmidrule{1-6}
\multirow{16}{*}{APT} & \multirow{4}{*}{Walker} & Flip & $353$ {\tiny $\pm 37$} & $58$ {\tiny $\pm 28$} & $280$ {\tiny $\pm 26$} \\
& & Run & $239$ {\tiny $\pm 34$} & $47$ {\tiny $\pm 17$} & $160$ {\tiny $\pm 25$} \\
& & Stand & $768$ {\tiny $\pm 110$} & $257$ {\tiny $\pm 82$} & $486$ {\tiny $\pm 16$} \\
& & Walk & $504$ {\tiny $\pm 79$} & $66$ {\tiny $\pm 30$} & $514$ {\tiny $\pm 71$} \\
\cmidrule{2-6}
& \multirow{4}{*}{Cheetah} & Run & $251$ {\tiny $\pm 18$} & $28$ {\tiny $\pm 13$} & $326$ {\tiny $\pm 67$} \\
& & Run Backward & $169$ {\tiny $\pm 109$} & $24$ {\tiny $\pm 15$} & $249$ {\tiny $\pm 83$} \\
& & Walk & $737$ {\tiny $\pm 54$} & $107$ {\tiny $\pm 49$} & $880$ {\tiny $\pm 60$} \\
& & Walk Backward & $578$ {\tiny $\pm 137$} & $109$ {\tiny $\pm 67$} & $839$ {\tiny $\pm 99$} \\
\cmidrule{2-6}
& \multirow{4}{*}{Quadruped} & Jump & $270$ {\tiny $\pm 40$} & $187$ {\tiny $\pm 44$} & $207$ {\tiny $\pm 62$} \\
& & Run & $188$ {\tiny $\pm 38$} & $121$ {\tiny $\pm 28$} & $125$ {\tiny $\pm 34$} \\
& & Stand & $375$ {\tiny $\pm 112$} & $242$ {\tiny $\pm 49$} & $267$ {\tiny $\pm 60$} \\
& & Walk & $173$ {\tiny $\pm 36$} & $127$ {\tiny $\pm 41$} & $124$ {\tiny $\pm 32$} \\
\cmidrule{2-6}
& \multirow{4}{*}{Jaco} & Reach Bottom Left & $17$ {\tiny $\pm 9$} & $24$ {\tiny $\pm 24$} & $18$ {\tiny $\pm 6$} \\
& & Reach Bottom Right & $31$ {\tiny $\pm 32$} & $22$ {\tiny $\pm 20$} & $23$ {\tiny $\pm 4$} \\
& & Reach Top Left & $42$ {\tiny $\pm 8$} & $97$ {\tiny $\pm 66$} & $27$ {\tiny $\pm 4$} \\
& & Reach Top Right & $64$ {\tiny $\pm 27$} & $37$ {\tiny $\pm 32$} & $39$ {\tiny $\pm 24$} \\
\cmidrule{1-6}
\multirow{16}{*}{Proto} & \multirow{4}{*}{Walker} & Flip & $267$ {\tiny $\pm 65$} & $85$ {\tiny $\pm 56$} & $140$ {\tiny $\pm 65$} \\
& & Run & $212$ {\tiny $\pm 44$} & $48$ {\tiny $\pm 19$} & $108$ {\tiny $\pm 35$} \\
& & Stand & $854$ {\tiny $\pm 46$} & $282$ {\tiny $\pm 164$} & $533$ {\tiny $\pm 132$} \\
& & Walk & $563$ {\tiny $\pm 268$} & $88$ {\tiny $\pm 61$} & $347$ {\tiny $\pm 71$} \\
\cmidrule{2-6}
& \multirow{4}{*}{Cheetah} & Run & $87$ {\tiny $\pm 67$} & $11$ {\tiny $\pm 12$} & $116$ {\tiny $\pm 31$} \\
& & Run Backward & $41$ {\tiny $\pm 22$} & $5$ {\tiny $\pm 5$} & $170$ {\tiny $\pm 78$} \\
& & Walk & $150$ {\tiny $\pm 126$} & $32$ {\tiny $\pm 51$} & $410$ {\tiny $\pm 303$} \\
& & Walk Backward & $384$ {\tiny $\pm 232$} & $26$ {\tiny $\pm 31$} & $743$ {\tiny $\pm 267$} \\
\cmidrule{2-6}
& \multirow{4}{*}{Quadruped} & Jump & $182$ {\tiny $\pm 22$} & $150$ {\tiny $\pm 41$} & $210$ {\tiny $\pm 62$} \\
& & Run & $120$ {\tiny $\pm 21$} & $98$ {\tiny $\pm 17$} & $158$ {\tiny $\pm 71$} \\
& & Stand & $226$ {\tiny $\pm 65$} & $181$ {\tiny $\pm 29$} & $293$ {\tiny $\pm 134$} \\
& & Walk & $108$ {\tiny $\pm 39$} & $87$ {\tiny $\pm 21$} & $157$ {\tiny $\pm 50$} \\
\cmidrule{2-6}
& \multirow{4}{*}{Jaco} & Reach Bottom Left & $24$ {\tiny $\pm 22$} & $75$ {\tiny $\pm 29$} & $32$ {\tiny $\pm 19$} \\
& & Reach Bottom Right & $26$ {\tiny $\pm 21$} & $31$ {\tiny $\pm 31$} & $24$ {\tiny $\pm 18$} \\
& & Reach Top Left & $43$ {\tiny $\pm 25$} & $47$ {\tiny $\pm 16$} & $38$ {\tiny $\pm 10$} \\
& & Reach Top Right & $48$ {\tiny $\pm 21$} & $60$ {\tiny $\pm 37$} & $66$ {\tiny $\pm 19$} \\
\cmidrule{1-6}
\multirow{16}{*}{RND} & \multirow{4}{*}{Walker} & Flip & $282$ {\tiny $\pm 52$} & $62$ {\tiny $\pm 57$} & $232$ {\tiny $\pm 41$} \\
& & Run & $146$ {\tiny $\pm 60$} & $42$ {\tiny $\pm 25$} & $126$ {\tiny $\pm 8$} \\
& & Stand & $557$ {\tiny $\pm 99$} & $172$ {\tiny $\pm 111$} & $496$ {\tiny $\pm 73$} \\
& & Walk & $452$ {\tiny $\pm 52$} & $104$ {\tiny $\pm 82$} & $376$ {\tiny $\pm 52$} \\
\cmidrule{2-6}
& \multirow{4}{*}{Cheetah} & Run & $178$ {\tiny $\pm 41$} & $221$ {\tiny $\pm 15$} & $276$ {\tiny $\pm 46$} \\
& & Run Backward & $126$ {\tiny $\pm 9$} & $171$ {\tiny $\pm 123$} & $297$ {\tiny $\pm 46$} \\
& & Walk & $470$ {\tiny $\pm 182$} & $535$ {\tiny $\pm 251$} & $895$ {\tiny $\pm 33$} \\
& & Walk Backward & $441$ {\tiny $\pm 107$} & $535$ {\tiny $\pm 440$} & $927$ {\tiny $\pm 35$} \\
\cmidrule{2-6}
& \multirow{4}{*}{Quadruped} & Jump & $273$ {\tiny $\pm 66$} & $224$ {\tiny $\pm 149$} & $244$ {\tiny $\pm 122$} \\
& & Run & $192$ {\tiny $\pm 29$} & $158$ {\tiny $\pm 82$} & $148$ {\tiny $\pm 19$} \\
& & Stand & $374$ {\tiny $\pm 85$} & $347$ {\tiny $\pm 191$} & $327$ {\tiny $\pm 126$} \\
& & Walk & $199$ {\tiny $\pm 63$} & $162$ {\tiny $\pm 92$} & $163$ {\tiny $\pm 45$} \\
\cmidrule{2-6}
& \multirow{4}{*}{Jaco} & Reach Bottom Left & $20$ {\tiny $\pm 14$} & $62$ {\tiny $\pm 16$} & $25$ {\tiny $\pm 3$} \\
& & Reach Bottom Right & $16$ {\tiny $\pm 8$} & $49$ {\tiny $\pm 29$} & $31$ {\tiny $\pm 11$} \\
& & Reach Top Left & $40$ {\tiny $\pm 16$} & $53$ {\tiny $\pm 7$} & $29$ {\tiny $\pm 4$} \\
& & Reach Top Right & $38$ {\tiny $\pm 21$} & $70$ {\tiny $\pm 17$} & $35$ {\tiny $\pm 13$} \\
    \bottomrule
    \end{tabular}
    }
\end{table*}

\clearpage

\begin{figure}
    \centering
    \includegraphics[width=0.95\linewidth]{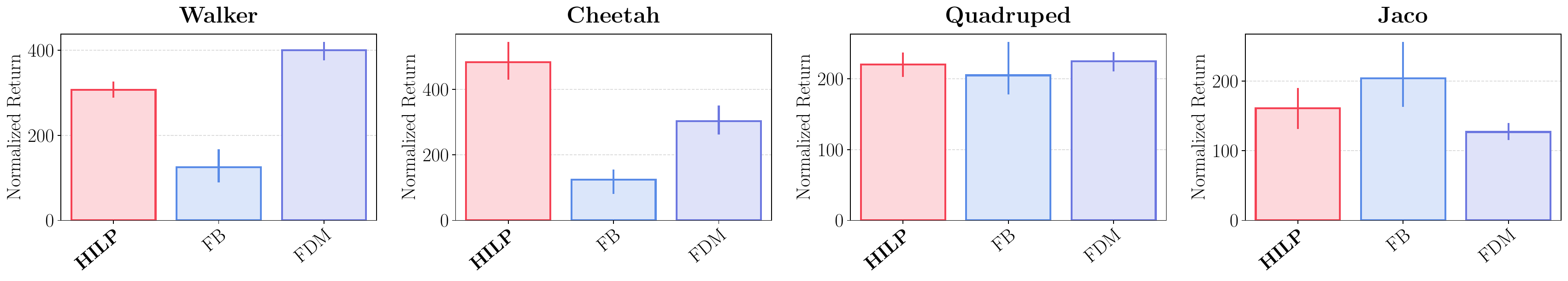}
    \vspace{-10pt}
    \caption{
    \footnotesize
    \textbf{Per-environment performances on the pixel-based ExORL benchmark.}
    The results are aggregated over $4$ tasks, $4$ datasets, and $4$ seeds (\ie, $64$ values in total).
    }
    \label{fig:zs_pixel_envs2}
\end{figure}
\begin{figure}
    \centering
    \includegraphics[width=0.95\linewidth]{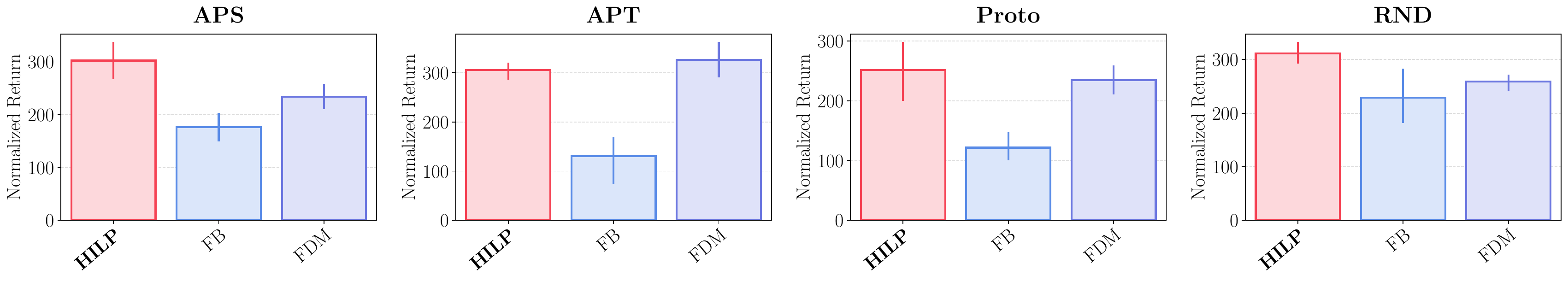}
    \vspace{-10pt}
    \caption{
    \footnotesize
    \textbf{Per-dataset performances on the pixel-based ExORL benchmark.}
    The results are aggregated over $4$ environments, $4$ tasks, and $4$ seeds (\ie, $64$ values in total).
    }
    \label{fig:zs_pixel_urls}
\end{figure}
\begin{figure}
    \centering
    \includegraphics[width=0.95\linewidth]{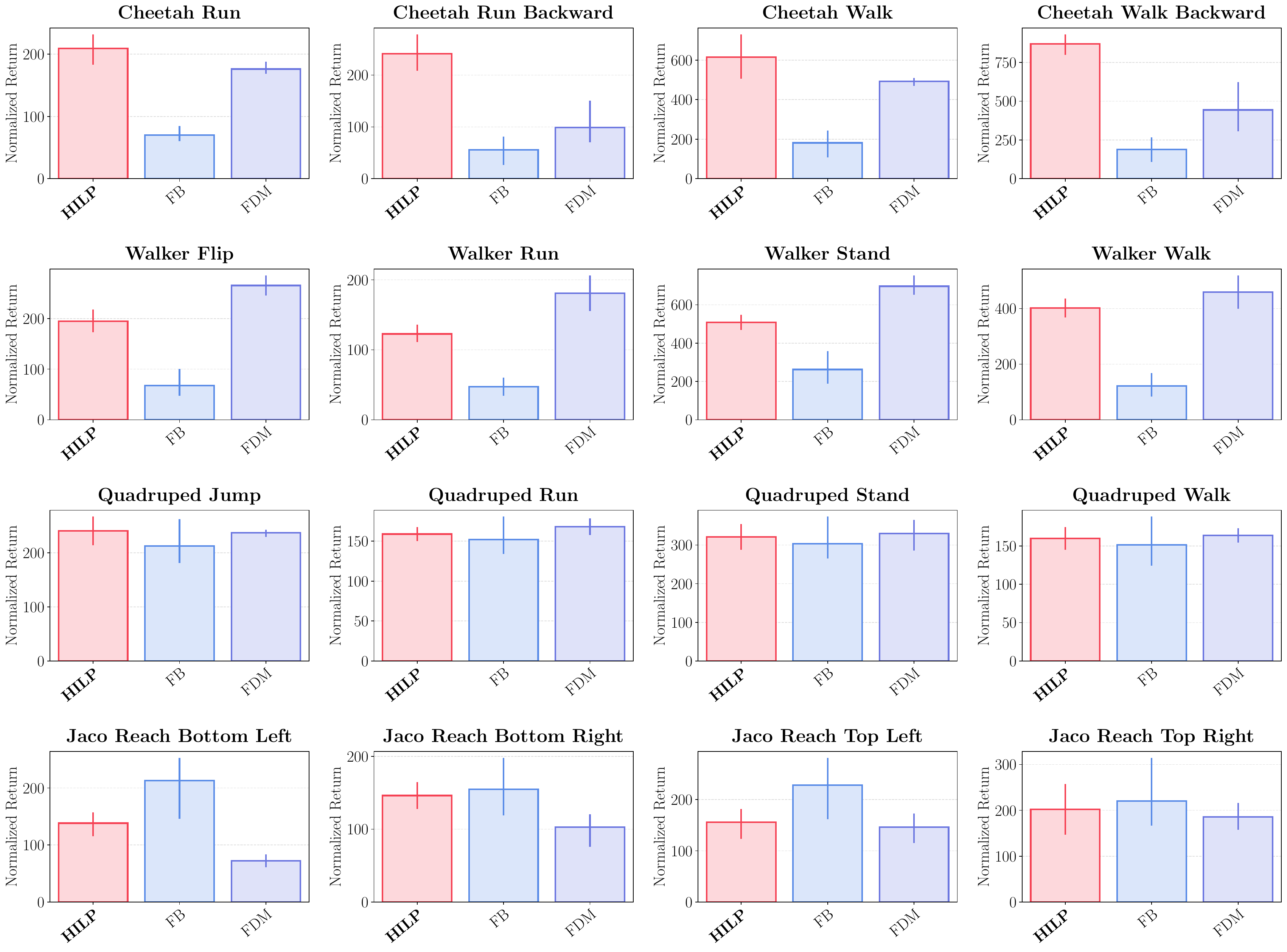}
    \vspace{-10pt}
    \caption{
    \footnotesize
    \textbf{Per-task performances on the pixel-based ExORL benchmark.}
    The results are aggregated over $4$ datasets and $4$ seeds (\ie, $16$ values in total).
    }
    \label{fig:zs_pixel_tasks}
\end{figure}

\clearpage

\section{Theoretical Results}
\label{sec:theory}

Let $d^*: \gS \times \gS \to \sR$ be the optimal temporal distance function.
Let $\phi: \gS \to \gZ$ be a representation function that maps states into a real Hilbert space
with an inner product $\langle \cdot, \cdot \rangle$ and its induced norm $\| \cdot \|$.
The role of $\phi$ is
to embed temporal distances into the latent Hilbert space such that $d^*(s, g) \approx \|\phi(s) - \phi(g)\|$.
Since we assume a deterministic MDP,
we denote the transition dynamics function and policies as deterministic functions:
$p: \gS \times \gA \to \gS$ and $\pi: \gS \to \gA$.
In this section, we will show that moving in the direction of $\phi(g) - \phi(s)$
is optimal to reach the goal $g$ from the state $s$ if embedding errors are sufficiently small.

For a state $s$ and a goal $g$, we define the following functions:
\begin{align}
z'^*(s, g) &:= \phi(s) + \frac{\phi(g) - \phi(s)}{\|\phi(g) - \phi(s)\|}, \\
\hat \pi(s, g) &:= \argmax_{a \in \gA}
    \left\langle \phi(s') - \phi(s), \frac{\phi(g) - \phi(s)}{\|\phi(g) - \phi(s)\|} \right\rangle \quad \mathrm{s.t.} \quad s' = p(s, a), \  \|\phi(s) - \phi(s')\| \leq 1, \label{eq:gc_bound_policy} \\
\hat z'(s, g) &:= \phi(p(s, \hat \pi(s, g))).
\end{align}
We denote the neighborhood states of $s$ as $N(s) := \{p(s, a): a \in \gA\}$.
Intuitively, $z'^*(s, g)$ is the optimal latent point
that is a unit distance away in the goal direction from the current latent state,
and $\hat z'(s, g)$ is the optimal next latent state that maximizes the directional reward $\langle \phi(s') - \phi(s), (\phi(g) - \phi(s)) / \|\phi(g) - \phi(s)\| \rangle$,

The following theorem states a condition for the policy $\hat \pi(s, g)$ to be optimal at $(s, g)$.

\begin{theorem}
\label{thm:gc_bound}
For a state-goal pair $(s, g) \in \gS \times \gS$, $s \neq g$,
assume that the local embedding error is bounded as $\sup_{s' \in N(s) \cup \{s\}} |d^*(s', g) - \|\phi(s') - \phi(g)\|| \leq \eps_e$
and the directional movement error is bounded as $\|z'^*(s, g) - \hat z'(s, g)\| \leq \eps_d$.
If $\, 4\eps_e + \eps_d < 1$,
$\hat \pi(s, g)$ is guaranteed to be an optimal action at $(s, g)$.
\end{theorem}
\begin{proof}
Define $\hat s' := p(s, \hat \pi(s, g))$.
Since $s \neq g$, we know $d^*(s, g) \geq 1$.
To show that $\hat \pi(s, g)$ is an optimal action,
it suffices to show that the temporal distance toward the goal is reduced by $1$ when the agent moves from $s$ to $\hat s'$.
We bound the difference between $d^*(\hat s', g)$ and $d^*(s, g) - 1$ as follows:
\begin{align}
&\left|d^*(\hat s', g) - (d^*(s, g) - 1)\right| \\
\leq& \left|d^*(\hat s', g) - \|\phi(\hat s') - \phi(g)\|\right| + \left|\|\phi(\hat s') - \phi(g)\| - (\|\phi(s) - \phi(g)\| - 1)\right|
    + \left|(\|\phi(s) - \phi(g)\| - 1) - (d^*(s, g) - 1)\right| \\
\leq& \eps_e + \left|\|\phi(\hat s') - \phi(g)\| - (\|\phi(s) - \phi(g)\| - 1)\right| + \eps_e \\
=& 2\eps_e + \|\phi(\hat s') - \phi(g)\| - (\|\phi(s) - \phi(g)\| - 1) \label{eq:gc_bound1} \\
\leq& 2\eps_e + \|\phi(\hat s') - z'^*(s, g)\| + \|z'^*(s, g) - \phi(g)\| - (\|\phi(s) - \phi(g)\| - 1), \label{eq:gc_bound2}
\end{align}
where we use
\begin{align}
\|\phi(\hat s') - \phi(g)\| + 1
&\geq \|\phi(\hat s') - \phi(g)\| + \|\phi(s) - \phi(\hat s')\| \\
&\geq \|\phi(s) - \phi(g)\|
\end{align}
for \Cref{eq:gc_bound1}.
To bound \Cref{eq:gc_bound2}, we consider the following two cases.

\textbf{Case \#$\mathbf{1}$: $\|\phi(g) - \phi(s)\| \geq 1$.}
In this case, we have
\begin{align}
\|z'^*(s, g) - \phi(g)\|
&= \left\|\phi(s) + \frac{\phi(g) - \phi(s)}{\|\phi(g) - \phi(s)\|} - \phi(g)\right \| \\
&= \left|\|\phi(g) - \phi(s)\| - 1\right| \\
&= \|\phi(g) - \phi(s)\| - 1.
\end{align}

\textbf{Case \#$\mathbf{2}$: $\|\phi(g) - \phi(s)\| < 1$.}
Similarly, we have
\begin{align}
\|z'^*(s, g) - \phi(g)\|
&= 1 - \|\phi(g) - \phi(s)\|,
\end{align}
and thus
\begin{align}
&\|z'^*(s, g) - \phi(g)\| - (\|\phi(s) - \phi(g)\| - 1) \\
=& 2(1 - \|\phi(s) - \phi(g)\|) \\
\leq& 2(d^*(s, g) - \|\phi(s) - \phi(g)\|) \\
\leq& 2\eps_e.
\end{align}

As $\|z'^*(s, g) - \phi(g)\| - (\|\phi(s) - \phi(g)\| - 1)$ is bounded by $2\eps_e$ in both cases,
we have
\begin{align}
\left|d^*(\hat s', g) - (d^*(s, g) - 1)\right| &\leq 2\eps_e + \|\phi(\hat s') - z'^*(s, g)\| + \|z'^*(s, g) - \phi(g)\| - (\|\phi(s) - \phi(g)\| - 1) \\
&\leq 4\eps_e + \|\phi(\hat s') - z'^*(s, g)\| \\
&\leq 4\eps_e + \eps_d.
\end{align}

Since we have $4\eps_e + \eps_d < 1$ and $|d^*(\hat s', g) - (d^*(s, g) - 1)|$ is an integer,
$|d^*(\hat s', g) - (d^*(s, g) - 1)|$ must be zero and thus $\hat \pi(s, g)$ is an optimal action.
\end{proof}

A keen reader may notice that \Cref{thm:gc_bound} only depends on $z'^*(s, g)$ and $\hat z'(s, g)$,
and is agnostic to the \emph{objective} of the policy (\Cref{eq:gc_bound_policy}),
$\langle \phi(s') - \phi(s), (\phi(g) - \phi(s))/\|\phi(g) - \phi(s)\|\rangle$.
The following theorem justifies this directional objective:
namely, the policy objective in \Cref{eq:gc_bound_policy} finds the optimal next latent state $z'^*(s, g)$
if it is a feasible point.
\begin{theorem}
If $z'^*(s, g) \in \{\phi(s'): s' \in N(s), \|\phi(s) - \phi(s')\| \leq 1\}$, then $\hat z'(s, g) = z'^*(s, g)$.
\end{theorem}
\begin{proof}
Recall that $\hat \pi(s, g)$ is defined as
\begin{align}
\hat \pi(s, g) = \argmax_{a \in \gA}
    \left\langle \phi(s') - \phi(s), \frac{\phi(g) - \phi(s)}{\|\phi(g) - \phi(s)\|} \right\rangle \quad \mathrm{s.t.} \quad s' = p(s, a), \  \|\phi(s) - \phi(s')\| \leq 1. \\
\end{align}
We have
\begin{align}
&\max_{s' \in N(s)}
    \left\langle \phi(s') - \phi(s), \frac{\phi(g) - \phi(s)}{\|\phi(g) - \phi(s)\|} \right\rangle \quad \mathrm{s.t.} \quad \|\phi(s) - \phi(s')\| \leq 1 \label{eq:reachable} \\
\leq &\max_{s' \in \gS}
    \left\langle \phi(s') - \phi(s), \frac{\phi(g) - \phi(s)}{\|\phi(g) - \phi(s)\|} \right\rangle \quad \mathrm{s.t.} \quad \|\phi(s) - \phi(s')\| \leq 1 \\
\leq &\max_{s' \in \gS}
    \| \phi(s') - \phi(s) \| \left \| \frac{\phi(g) - \phi(s)}{\|\phi(g) - \phi(s)\|} \right\| \quad \mathrm{s.t.} \quad \|\phi(s) - \phi(s')\| \leq 1 \\
\leq &1,
\end{align}
by the Cauchy-Schwarz inequality.
By the assumption, $z'^*(s, g)$ is a feasible point
and
\begin{align}
&\left\langle z'^*(s, g) - \phi(s), \frac{\phi(g) - \phi(s)}{\|\phi(g) - \phi(s)\|} \right\rangle \\
=&\left\langle \phi(s) + \frac{\phi(g) - \phi(s)}{\|\phi(g) - \phi(s)\|} - \phi(s), \frac{\phi(g) - \phi(s)}{\|\phi(g) - \phi(s)\|} \right\rangle \\
=&\left\langle \frac{\phi(g) - \phi(s)}{\|\phi(g) - \phi(s)\|}, \frac{\phi(g) - \phi(s)}{\|\phi(g) - \phi(s)\|} \right\rangle \\
=&1
\end{align}
holds.
Hence, the maximum in \Cref{eq:reachable} is attainable by $z'^*(s, g)$,
and thus $\hat z'(s, g) = z'^*(s, g)$ holds.
\end{proof}

Finally, as a corollary to \Cref{thm:gc_bound}, we have the following:
\begin{corollary}
\label{thm:gc_total_bound}
If embedding errors are bounded as $\sup_{s, g \in \gS} |d^*(s, g) - \|\phi(s) - \phi(g)\|| \leq \eps_e$,
directional movement errors are bounded as $\sup_{s, g \in \gS} \|z'^*(s, g) - \hat z'(s, g)\| \leq \eps_d$,
and $\, 4\eps_e + \eps_d < 1$,
then $\hat \pi(s, g)$ is an optimal goal-reaching policy.
\end{corollary}

Intuitively, \Cref{thm:gc_total_bound} tells us that if the embedding error is small enough,
directional movements in the latent space are optimal for solving goal-reaching tasks.

\textbf{Limitations.}
One natural question to ask is whether it is always possible to embed any MDP into a Hilbert space up to arbitrary accuracy.
Unfortunately, this is not always possible.
First, temporal distances are asymmetric but the distance metric of the Hilbert space $\gZ$ is symmetric.
Second, even when the environment is completely symmetric,
there exists a symmetric MDP that is not embeddable
into a Hilbert space with an arbitrarily low approximation error~\citep{dist_indyk2017,dist_pitis2020}.
This is mainly because Hilbert spaces are highly structured,
especially compared to metric or quasimetric spaces, which do not require a well-defined inner product.
Nonetheless, the inner product structure of the Hilbert space naturally enables useful prompting strategies
for zero-shot RL and goal-conditioned RL,
and we empirically found that even MDPs that in principle do not have a lossless Hilbert representation
can still be solved effectively via our method in our experiments.
We believe finding a way to relax the Hilbert condition while having similar prompting and planning strategies
is an exciting and important future research direction.

\section{Experimental Details}
\label{sec:exp_detail}

We implement HILPs based on two different codebases:
the official implementation of FB representations~\citep{zs_touati2023} for zero-shot RL experiments
and that of HIQL~\citep{hiql_park2023} for offline goal-conditioned RL and hierarchical RL experiments.
Our implementations are publicly available at the following repository: \hilpcode.
We run our experiments on an internal cluster consisting of A5000 GPUs.
Each run in this work takes no more than $28$ hours.

\subsection{Environments and Datasets}

\textbf{ExORL~\citep{exorl_yarats2022}.}
The ExORL benchmark consists of a set of datasets collected by unsupervised RL agents~\citep{urlb_laskin2021}
on the DeepMind Control Suite~\citep{dmc_tassa2018}.
We use four environments (Walker (\Cref{fig:envs}a), Cheetah (\Cref{fig:envs}b),
Quadruped (\Cref{fig:envs}c), and Jaco (\Cref{fig:envs}d))
and four datasets collected by APS~\citep{aps_liu2021}, APT~\citep{apt_liu2021},
Proto~\citep{protorl_yarats2021}, and RND~\citep{rnd_burda2019} in each environment.
Following \citet{zs_touati2023}, we use the first $5$M transitions from each dataset.
Each environment has four test-time tasks:
Walker has Flip, Run, Stand, and Walk;
Cheetah has Run, Run Backward, Walk, and Walk Backward;
Quadruped has Jump, Run, Stand, Walk;
Jaco has Reach Bottom Left, Reach Bottom Right, Reach Top Left, and Reach Top Right.
Among the four environments,
Walker, Cheetah, and Quadruped have a maximum return of $1000$,
and Jaco has a maximum return of $250$.
As such, we multiply Jaco returns by $4$ to normalize them for aggregation.
For pixel-based ExORL experiments,
we convert each state in the datasets into a $64 \times 64 \times 3$-sized camera image by rendering it.

\textbf{AntMaze~\citep{d4rl_fu2020}.}
The AntMaze datasets from D4RL~\citep{d4rl_fu2020} consist of trajectories of a quadrupedal robot
navigating through a maze from random locations to other locations.
We employ the two most challenging datasets with the largest maze (``antmaze-large-\{diverse, play\}-v2'', \Cref{fig:envs}e)
from the original D4RL benchmark,
and two even larger settings (``antmaze-ultra-\{diverse, play\}-v0'', \Cref{fig:envs}f) introduced by \citet{tap_jiang2023},
where the ``ultra'' maze is twice the size of the ``large'' maze.
For goal-conditioned RL experiments in \Cref{sec:exp_zs_gcrl},
we use the same goal-conditioned evaluation setting as \citet{hiql_park2023}:
we specify the test-time goal $g$ by concatenating the $x$-$y$ coordinates of the original target goal
to the proprioceptive state dimensions of the first observation in the dataset.
The agent gets a reward of $1$ when it reaches the target goal.
In \Cref{table:gcrl,table:hrl}, we multiply the returns by $100$ to normalize them.
For hierarchical RL experiments in \Cref{sec:exp_hrl}, we use the original non-goal-conditioned tasks.

\textbf{Kitchen~\citep{ril_gupta2019,d4rl_fu2020}.}
The Kitchen datasets from D4RL~\citep{d4rl_fu2020} consist of trajectories
of a robotic arm manipulating different kitchen objects in various orders
in the Kitchen environment~\citep{ril_gupta2019} (\Cref{fig:envs}g).
We employ two datasets (``kitchen-\{partial, mixed\}-v0'') from the original D4RL benchmark.
For goal-conditioned RL experiments in \Cref{sec:exp_zs_gcrl},
we use the same goal-conditioned evaluation setting as \citet{hiql_park2023}:
we specify the test-time goal $g$ by concatenating the proprioceptive state dimensions of the first observation in the dataset
to the object states of the target goal given by the environment.
The agent gets a reward of $1$ whenever it achieves a subtask,
where each task consists of a total of four subtasks.
In \Cref{table:gcrl,table:hrl}, we multiply the returns by $25$ to normalize them.
For hierarchical RL experiments in \Cref{sec:exp_hrl}, we use the original non-goal-conditioned tasks.
For pixel-based Kitchen experiments,
we convert each state in the datasets into a $64 \times 64 \times 3$-sized camera image by rendering it.
We use the same camera configuration as \citet{lexa_mendonca2021,metra_park2024} (\Cref{fig:envs}g).

\subsection{Implementation Details}

\textbf{Hilbert representations.}
In \Cref{eq:obj_phi},
we use the same goal relabeling strategy as \citet{hiql_park2023} except that we do not set $g = s$,
since $V(s, s) = 0$ is always guaranteed in our parameterization.
Namely, we sample $g$ either from a geometric distribution over the future states within the same trajectory (with probability $0.625$),
or uniformly from the dataset (with probability $0.375$).
We note that the values $0.625$ and $0.375$ come from the original hyperparameters used by \citet{hiql_park2023}
(which use $g=s$ with probability $0.2$, future states with probability $0.5$, and random states with probability $0.3$),
where we redistribute the unnecessary probability mass of $g=s$ across the other two bins.
To avoid numerical issues with gradient descent, we add a small value ($\eps = 10^{-6}$)
when computing $\|\phi(s) - \phi(g)\|$.

\textbf{Zero-shot RL (\Cref{sec:exp_zs_rl}).}
We evaluate HILPs and all baselines on the same codebase built on
the official implementation of the work by \citet{zs_touati2023}.
For HILP, we use the centered reward function introduced in \Cref{sec:policy}
and the zero-shot prompting scheme introduced in \Cref{sec:zs_rl}.
For HILP-G in Jaco, we use the reward function in \Cref{eq:hilp_reward}
and the goal-conditioned prompting scheme introduced in \Cref{sec:zs_gcrl},
where the goal is specified as the state with the highest reward value from the offline dataset.
For the FB, SF, and GCRL baselines, we follow the implementations provided by \citet{zs_touati2023}.
We use TD3~\citep{td3_fujimoto2018} as the base (offline) RL algorithm to train these methods.
Following \citet{zs_touati2023},
for the HILP, FB, and SF methods,
we either sample a latent vector $z$ uniformly from the prior distribution (with probability $0.5$)
or set $z$ to the latent vector that corresponds to a goal-reaching task (with probability $0.5$).
For successor feature losses, we use either the vector loss or the Q loss~\citep{usf_ma2020},
depending on the environment.
For hyperparameter tuning,
we individually tune HILP, FB, Lap (the best SF method reported in the work by \citet{zs_touati2023}), and GC-TD3
in each environment with the RND dataset,
and apply the found hyperparameters to the other datasets and to the other methods in the same category.
We report the full list of the hyperparameters used in our zero-shot RL experiments in \Cref{table:hyp_zs_rl}.

\textbf{Offline goal-conditioned RL (\Cref{sec:exp_zs_gcrl}).}
We implement HILP on top of the official codebase of the work by \citet{hiql_park2023}.
For HILP, we use the reward function in \Cref{eq:hilp_reward}
and the goal-conditioned prompting scheme introduced in \Cref{sec:zs_gcrl}.
We use IQL~\citep{iql_kostrikov2022} with AWR~\citep{awr_peng2019} as an offline algorithm to train policies.
For HILP-Plan, at each evaluation epoch,
we first randomly sample $N=50000$ states $w_1, w_2, \dots, w_N$ from the dataset $\gD$,
and pre-compute their representations $\phi(w_1), \phi(w_2), \dots, \phi(w_N)$.
Then, at every time step, we find the $\argmin$ of \Cref{eq:plan} over the $N$ samples
using the pre-computed representations.
In practice, we use the average of the $50$ $\argmin$ representations, as we found this to lead to better performance.
For GC-IQL and GC-BC, we use the implementations provided by \citet{hiql_park2023}.
They are implemented on the same codebase as HILP.
For GC-CQL, we modify the JaxCQL repository~\citep{jaxcql_geng2022}
to make it compatible with our goal-conditioned setting.
We mostly follow the hyperparameters used by \citet{calql_nakamoto2023}.
We use the same goal relabeling strategy as \citet{hiql_park2023} for all three goal-conditioned RL methods.
For FB, we use the official implementation provided by \citet{zs_touati2023},
where we additionally implement D4RL environments.
For SF methods (FDM, Lap, and Rand), we re-implement IQL versions of them on the same codebase as HILP,
as we found these versions to perform better than the original implementations by \citet{zs_touati2023}.
Among FB and SF methods,
we only re-implement SF methods based on IQL,
as FB in its current form is not directly compatible with IQL.
We report the full list of the hyperparameters used in our offline goal-conditioned RL experiments in \Cref{table:hyp_zs_gcrl}.

\textbf{Hierarchical RL (\Cref{sec:exp_hrl}).}
We implement hierarchical HILP and OPAL on top of the official codebase of the work by \citet{hiql_park2023}.
To train a high-level policy $\pi^h(z \mid s)$ on top of our latent-conditioned (low-level) policy $\pi(a \mid s, z)$,
we first sample $(s_t, s_{t+k})$ tuples from the dataset,
label them with \mbox{$z = (\phi(s_{t+k}) - \phi(s_t)) / \|\phi(s_{t+k}) - \phi(s_t)\|$},
and use $z$ as high-level actions.
$k$ is a hyperparameter that determines the high-level action length.
For OPAL,
we use our own implementation on top of the same codebase as HILP,
as we were unable to find the official implementation.
We sample trajectory chunks $(s_{t:t+k}, a_{t:t+k-1})$ from the dataset
and train a trajectory VAE consisting of three components:
a trajectory encoder $p(z \mid s_{t:t+k}, a_{t:t+k-1})$ modeled by a bi-directional GRU~\citep{gru_cho2014},
a decoder parameterized as $\pi(a_t \mid s_t, z)$,
and a prior $p(z \mid s_t)$.
For both HILP and OPAL, to ensure a fair comparison, we use the same offline RL algorithm (IQL) for high-level policy learning.
We report the full list of the hyperparameters used in our zero-shot hierarchical RL experiments in \Cref{table:hyp_hrl}.

\begin{table}[t]
    \caption{
    \footnotesize
    \textbf{Hyperparameters for zero-shot RL.}
    }
    \label{table:hyp_zs_rl}
    \vspace{5pt}
    \begin{center}
    \begin{tabular}{lc}
        \toprule
        Hyperparameter & Value \\
        \midrule
        \# gradient steps & $10^6$ (state-based), $5 \times 10^5$ (pixel-based) \\
        Learning rate & $0.0005$ ($\phi$), $0.0001$ (others) \\
        Optimizer & Adam~\citep{adam_kingma2015} \\
        Minibatch size & $1024$ (state-based), $512$ (pixel-based) \\
        MLP dimensions & $(512, 512)$ ($\phi$), $(1024, 1024, 1024)$ (others)%
        \tablefootnote{
        Following \citet{zs_touati2023}, for policies $\pi(a \mid s, z)$ and TD3 values $Q(s, a, z)$,
        we process $s$ (or $(s, a)$) and $(s, z)$ separately with $(1024, 512)$-sized MLPs,
        concatenate them together, and then pass another $(1024)$-sized MLP.
        } \\
        TD3 target smoothing coefficient & $0.01$ \\
        TD3 discount factor $\gamma$ & $0.98$ \\
        Latent dimension & $50$ \\
        \# state samples for latent vector inference & $10000$ \\
        Successor feature loss & Q loss (\{HILP, SF\} on \{Quadruped, Jaco\}), vector loss (others) \\
        Hilbert representation discount factor & $0.96$ (Walker), $0.98$ (others) \\
        Hilbert representation expectile & $0.5$ (HILP), $0.9$ (HILP-G) \\
        Hilbert representation target smoothing coefficient & $0.005$ \\
        \bottomrule
    \end{tabular}
    \end{center}
\end{table}

\begin{table}[t]
    \caption{
    \footnotesize
    \textbf{Hyperparameters for offline goal-conditioned RL.}
    }
    \label{table:hyp_zs_gcrl}
    \vspace{5pt}
    \begin{center}
    \begin{tabular}{lc}
        \toprule
        Hyperparameter & Value \\
        \midrule
        \# gradient steps & $10^6$ (AntMaze), $5 \times 10^5$ (Kitchen) \\
        Learning rate & $0.0003$ \\
        Optimizer & Adam~\citep{adam_kingma2015} \\
        Minibatch size & $1024$ (state-based), $256$ (pixel-based) \\
        Value MLP dimensions & $(512, 512, 512)$ \\
        Policy MLP dimensions & \makecell{$(256, 256)$ (\{GC-BC, GC-IQL\} on AntMaze-Large)%
        \tablefootnote{
        We found that $(256, 256)$-sized policy networks lead to better performance than $(512, 512, 512)$-sized ones
        for GC-BC and GC-IQL on AntMaze-Large.
        }, \\ $(512, 512, 512)$ (others)} \\
        Target smoothing coefficient & $0.005$ \\
        Discount factor $\gamma$ & $0.99$ \\
        Latent dimension & $32$ \\
        Hilbert representation discount factor & $0.99$ \\
        Hilbert representation expectile & $0.7$ (Visual Kitchen), $0.95$ (others) \\
        Hilbert representation target smoothing coefficient & $0.005$ \\
        HILP IQL expectile & $0.7$ (Visual Kitchen), $0.9$ (others) \\
        HILP AWR temperature & $10$ \\
        \bottomrule
    \end{tabular}
    \end{center}
\end{table}

\begin{table}[t]
    \caption{
    \footnotesize
    \textbf{Hyperparameters for hierarchical RL.}
    }
    \label{table:hyp_hrl}
    \vspace{5pt}
    \begin{center}
    \begin{tabular}{lc}
        \toprule
        Hyperparameter & Value \\
        \midrule
        \# gradient steps & $5 \times 10^5$ \\
        Learning rate & $0.0003$ \\
        Optimizer & Adam~\citep{adam_kingma2015} \\
        Minibatch size & $1024$ (HILP), $256$ (OPAL, high-level IQL) \\
        Value MLP dimensions & $(512, 512, 512)$ \\
        Policy MLP dimensions & $(512, 512, 512)$ (HILP), $(256, 256)$ (high-level IQL) \\
        Target smoothing coefficient & $0.005$ \\
        Discount factor $\gamma$ & $0.99$ \\
        Latent dimension & $32$ (HILP), $8$ (OPAL)%
        \tablefootnote{
        We found that OPAL works better with $8$-dimensional latent spaces (as in the original work), compared to $32$-dimensional ones (as in this work), especially in AntMaze tasks.
        }
        \\
        Hilbert representation discount factor & $0.99$ \\
        Hilbert representation expectile & $0.95$ (AntMaze), $0.7$ (Kitchen) \\
        Hilbert representation target smoothing coefficient & $0.005$ \\
        HILP IQL expectile & $0.9$ \\
        HILP AWR temperature & $10$ \\
        OPAL VAE MLP dimensions & $(256, 256)$ \\
        OPAL VAE \# GRU layers & $2$ \\
        OPAL VAE KL coefficient & $0.1$ \\
        High-level action length $k$ & $10$ \\
        High-level IQL discount factor & $0.99$ \\
        High-level IQL expectile & $0.9$ (AntMaze), $0.7$ (Kitchen) \\
        High-level AWR temperature & $1$ \\
        High-level value normalization & None (AntMaze), LayerNorm~\citep{ln_ba2016} (Kitchen) \\
        \bottomrule
    \end{tabular}
    \end{center}
\end{table}

\end{document}